\documentclass[10pt,twocolumn,letterpaper]{article}

\usepackage[pagenumbers]{cvpr} 
\usepackage{bm,stfloats}
\usepackage[dvipsnames]{xcolor}

\definecolor{cvprblue}{rgb}{0.21,0.49,0.74}
\usepackage[pagebackref,breaklinks,color links,citecolor=cvprblue]{hyperref}

\usepackage{xcolor}
\usepackage{amsthm}
\usepackage{multicol}
\usepackage{multirow}
\usepackage{comment}
\newcommand{\piou}{\text{ProbIoU}}

\newcommand{\upper}[1]{{\color{Green}#1}}
\def\paperID{6698} 
\def\confName{CVPR}
\def\confYear{2025}
\newtheorem{theorem}{Theorem}[section]
\newtheorem{proposition}[theorem]{Proposition}
\title{GauCho: Gaussian Distributions with Cholesky Decomposition for Oriented Object Detection}

\author{Jeffri Murrugarra-LLerena*\\
Computer Science Department\\
Stony Brook University\\
{\tt\small jmurrugarral@cs.stonybrook.edu}
\and
José Henrique Lima Marques*\\
Institute of Informatics\\
Federal University of Rio Grande do Sul\\
{\tt\small jhlmarques@inf.ufrgs.br}
\and
Claudio R. Jung\\
Institute of Informatics\\
Federal University of Rio Grande do Sul\\
{\tt\small crjung@inf.ufrgs.br}
}

\begin{document}
\maketitle
\begin{abstract}
Oriented Object Detection (OOD) has received increased attention in the past years, being a suitable solution for detecting elongated objects in remote sensing analysis. In particular, using regression loss functions based on Gaussian distributions has become attractive since they yield simple and differentiable terms. However, existing solutions are still based on regression heads that produce Oriented Bounding Boxes (OBBs), and the known problem of angular boundary discontinuity persists. In this work, we propose a regression head for OOD that directly produces Gaussian distributions based on the Cholesky matrix decomposition. The proposed head, named GauCho, theoretically mitigates the boundary discontinuity problem and is fully compatible with recent Gaussian-based regression loss functions. Furthermore, we advocate using Oriented Ellipses (OEs) to represent oriented objects, which relates to GauCho through a bijective function and alleviates the encoding ambiguity problem for circular objects. Our experimental results show that GauCho can be a viable alternative to the traditional OBB head, achieving results comparable to or better than state-of-the-art detectors for the challenging dataset DOTA. Our code will be available at \url{https://github.com/jhlmarques/mmrotate-gaucho}.
\end{abstract}

\section{Introduction}

Oriented object detection (OOD) is an essential application of computer vision that extends traditional object detection by considering the orientation of objects in images, being particularly useful in remote sensing~\cite{leng2024recent}. Standard object detection methods use \textit{horizontal bounding boxes} (HBBs), and the \textit{de facto} representations for OOD are rotated HBBs called \textit{oriented bounding boxes} (OBBs). 

In terms of network architecture design, we can modify the regression head of a traditional object detector that produces the center $(x,y)$ and dimensions $(w,h)$ of an HBB by adding an angular parameter $\theta$ related to the orientation of the OBB. For HBB detection, each box is characterized by a unique pair of values $(w,h)$, with $w,h >0$. However, there are several issues when defining the parameters of an OBB~\cite{yang2023detecting,yu2024boundary,xiao2024theoretically} regarding the shape parameters $(w,h,\theta)$. Common representations are the OpenCV (OC) and the long-edge (LE) representation~\cite{yang2023detecting}: OC defines the angle based on the OBB side that lies in $[-90^\circ, 0)$; LE defines the angle based on the largest side, so that $\theta \in[-90^\circ, 90^\circ)$. These parametrizations present a \textit{boundary discontinuity problem}, where different sets of parameters generate very similar OBBs. Hence, regression loss functions that use independent parameter-wise comparisons for center, dimensions, and angle (\eg, using the $L_1$ loss) might present large values for similar OBBs. As an example, consider an OBB with dimensions $w=1$, $h=3$ rotated by $30^\circ$. Its representation in the LE parameterization is $(3,1,30^\circ)$, where the first component is the dimension related to the angular information. Figure~\ref{fig:disc:le} shows the LE  parameters after rotating this OBB in the range $[-90^\circ, 90^\circ)$, which covers all possible rotations -- we show the angle and the corresponding OBB dimension, and can observe a discontinuity at $60^\circ$ rotation.

\begin{figure}[t]
    \centering
    \subfloat[][]{\includegraphics[height = 0.15\textwidth]{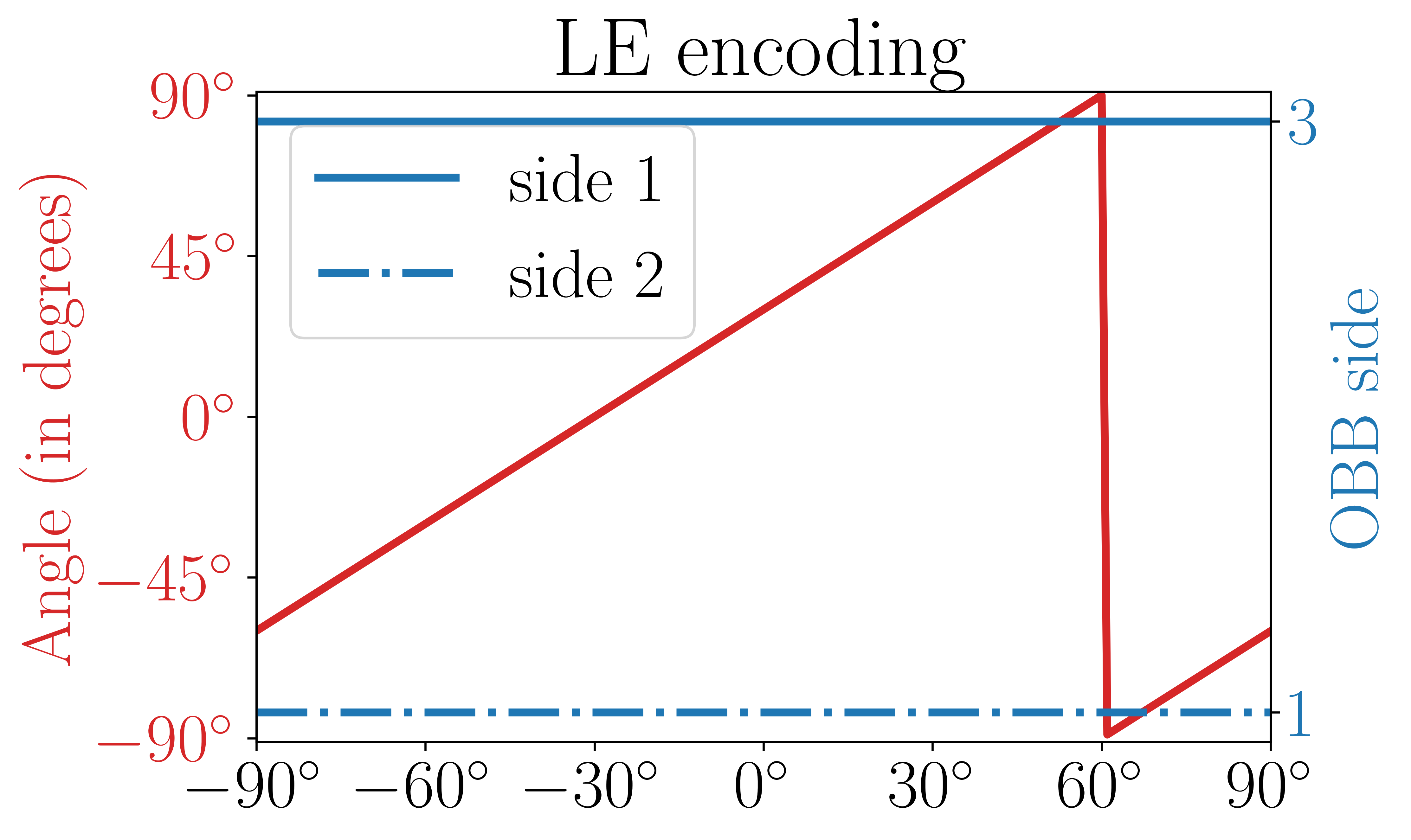}\label{fig:disc:le}}
        \subfloat[][]{\includegraphics[height = 0.15\textwidth]{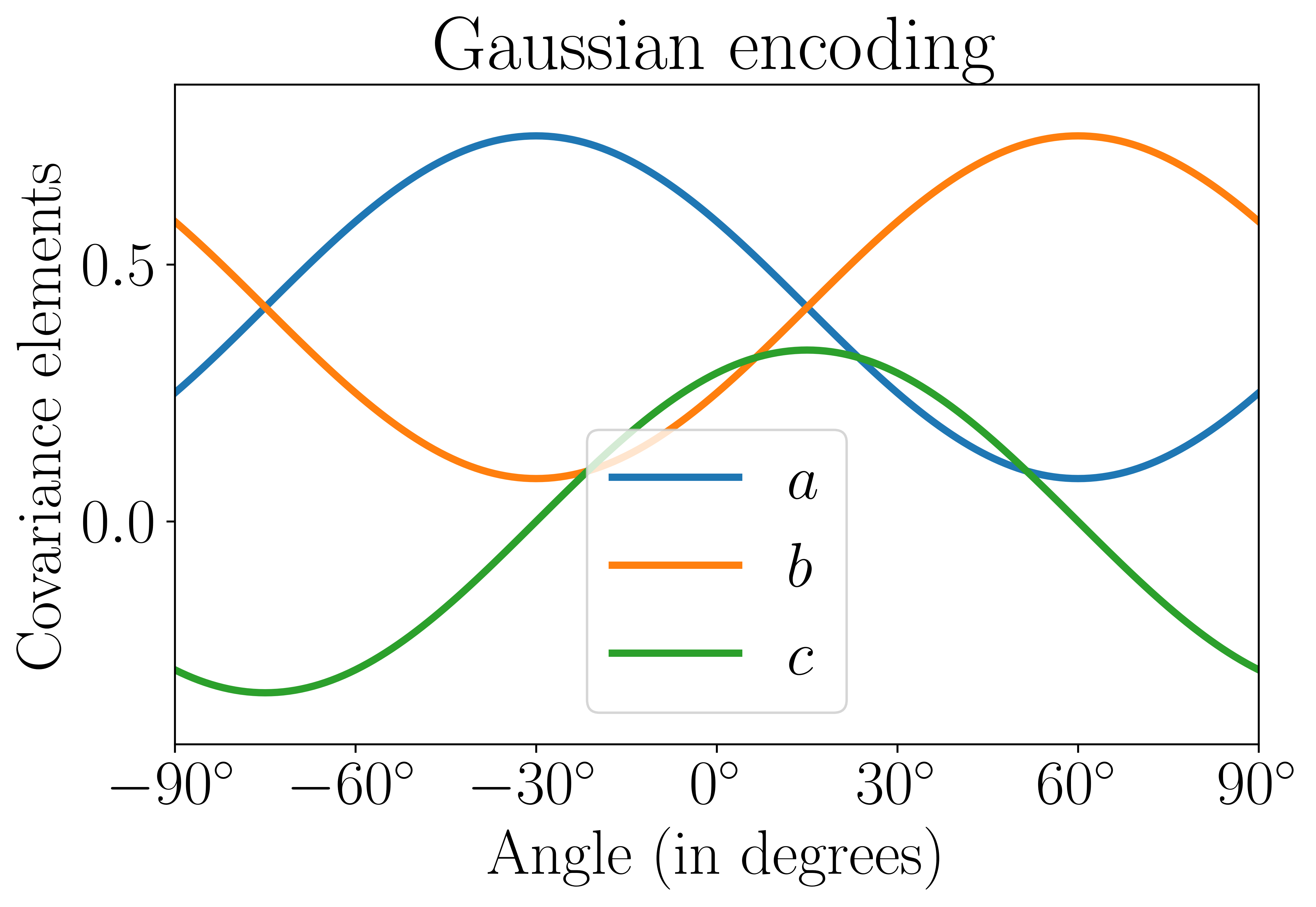}\label{fig:disc:gaussian}}
    \caption{\label{fig:disc:parametrizations}Examples of different parameterizations for OOD. (a) Long-Edge (LE) for OBBs. (b) Gaussian (covariance matrix).}
\end{figure}

One popular strategy for mitigating the discontinuity problem is using a \textit{holistic} regression loss function that optimizes the OBB parameters jointly. In particular, a promising approach consists of mapping OBBs to Gaussian distributions and defining regression terms based on distances between Gaussians~\cite{yang2021rethinking,yang2021learning,yang2023detecting,yang2023kfiou,murrugarra2024probabilistic}. Such terms are inherited from distance functions between generic probability density functions (PDFs), and they provide simple differentiable closed-form expressions for Gaussian PDFs that are suitable for training deep models. However, these approaches face two main problems. The first one is that the conversion from OBBs to Gaussians is not injective: square OBBs map to the same isotropic Gaussian regardless of the orientation~\cite{xiao2024theoretically,murrugarra2024probabilistic}, and angular information is lost in these cases -- this problem is called \textit{decoding ambiguity} by some authors~\cite{xiao2024theoretically}. The second problem, recently noted in~\cite{yu2024boundary,xu2024rethinking}, is that Gaussian-based loss functions might still suffer from the boundary discontinuity problem at inference time. On the other hand, Gaussian functions naturally mitigate the \textit{encoding ambiguity} problem for circular objects, recently pointed out in~\cite{murrugarra2024probabilistic}. For circular objects, any square with arbitrary rotation provides an equally good fit, as illustrated in Figure~\ref{fig:ra:re}. Although a canonical orientation could be selected (e.g., an axis-aligned square), any default choice would lead to inconsistencies for image augmentation based on rotations, as discussed in Section~\ref{sec:discussion}. Whereas \textit{any} representation based on OBBs suffers from the \textit{encoding ambiguity} problem, a Gaussian representation inherently solves the problem: squares with arbitrary rotations map to the same isotropic Gaussian, yielding a unique representation.

\begin{figure}[ht!]
    \centering
    \subfloat[][]{\includegraphics[height =  0.15\textwidth]{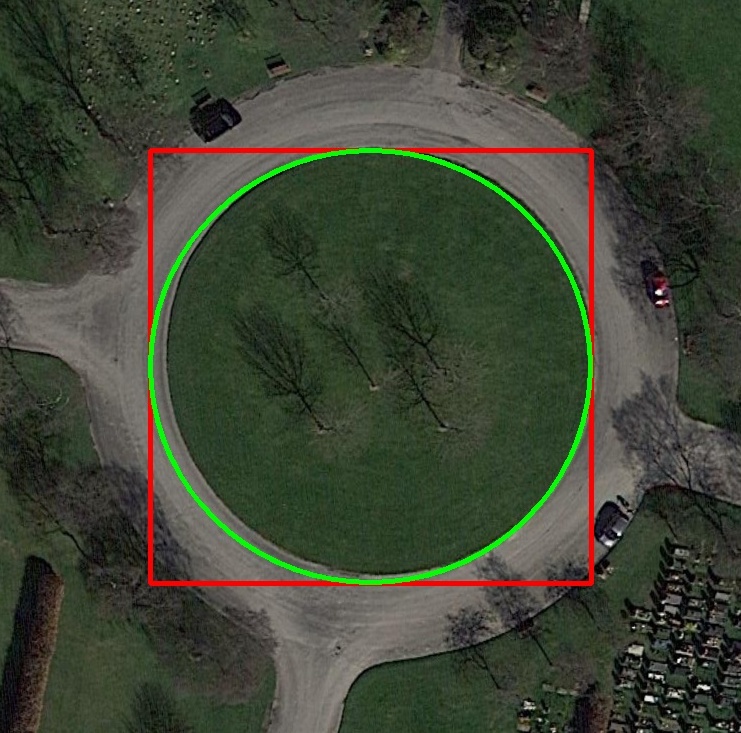}\label{fig:ra:angle0}}
    \subfloat[][]{\includegraphics[height =  0.15\textwidth]{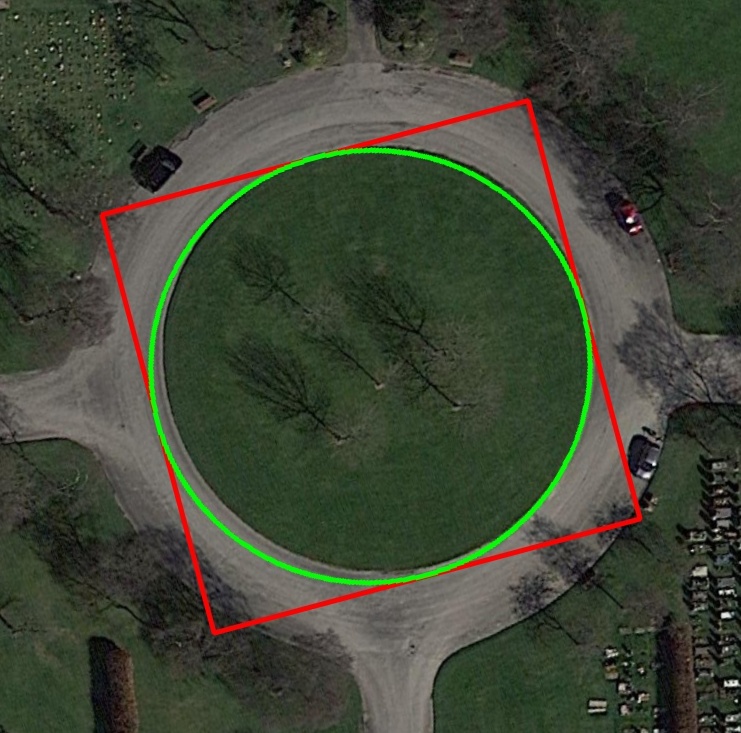} \label{fig:ra:angle15}}
    \subfloat[][]{\includegraphics[height =  0.15\textwidth]{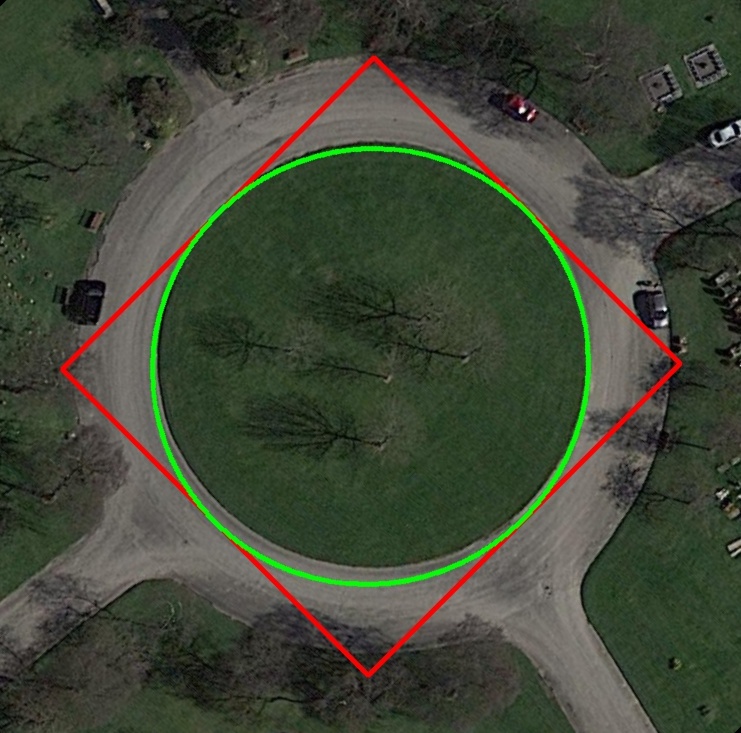}\label{fig:ra:angle45}}
    \caption{\label{fig:ra:re} \textit{Encoding ambiguity} problem for circular objects: any rotated square (red) is a viable choice. Oriented Ellipses (OEs, in green) induced by Gaussian representations mitigate the problem.  }
\end{figure}

In this work, we propose a new paradigm for OOD by regressing the parameters of a Gaussian distribution \textit{directly} from the network, avoiding the intermediate use of OBBs. The encoding through a covariance matrix provides a \textit{continuous} representation w.r.t. angular variations, as illustrated in Figure~\ref{fig:disc:gaussian}. Since covariance matrices are positive-definite (hence, their elements are not independent), we propose to regress the lower-triangular Cholesky decomposition of the covariance matrix, which provides a bijective mapping between the network output and the Gaussian representation. The proposed representation, called GauCho (\textbf{Gau}ssian-based \textbf{Cho}lesky representation), can be readily integrated into existing OOD detectors based on OBB representations and trained with any Gaussian-based loss function. Our experimental results show that the GauCho regression head presents competitive results compared to the traditional OBB head for several Gaussian-based loss functions and baseline OOD methods tested on different datasets, with a clear improvement in DOTA~\cite{Xia:2018:CVPR} when using the anchor-free detector FCOS~\cite{tian2019fcos}. The main contributions of this work are:
\newline
\noindent
$\bullet$ We present a novel regression head for OOD that is fully compatible with Gaussian loss functions and that \textit{theoretically} mitigates the angular discontinuity problem;
\newline
\noindent
$\bullet$ We show that the parameters in the Cholesky decomposition are directly related to the geometric parameters of the corresponding Gaussian/OBB, allowing adaptations for anchor-based and anchorless OOD approaches;
\newline
\noindent
$\bullet$ We show a one-to-one mapping between GauCho and oriented ellipses (OEs), and advocate their use as an alternative representation for oriented object detection, which particularly mitigates the \textit{encoding ambiguity} problem for circular objects.

\section{Related Work}

\noindent
\textbf{Oriented Object Detectors:} OOD architectures are analogous to traditional HBB object detectors, and the main difference is the regression of an additional parameter related to the orientation. They can rely on anchors and predict offsets and shape adjustment factors on top of predefined bounding boxes~\cite{lin2017retinanet}, or be anchor-free and directly regress the BB parameters~\citep{tian2019fcos,zhou2019objects,lang2021dafne}. They can also be categorized into single-stage methods, where the BBs are produced directly from the network~\cite{lin2017retinanet,tian2019fcos} sometimes with an additional refinement step~\cite{Xue2021r3det}, or two-stage, in which proposals are initially created and then refined in a second stage~\cite{ding2019roitransformer,han2021redet,xie2024orcnn_beyond}. Within these basic categories of OOD approaches, several improvements have been proposed to specific modules for enhancing the performance~\cite{han2021redet, ming2021cfc, han2022s2anet, xie2024orcnn_beyond, ming2023tioe, qiao2023grg,lee2024fred} by incorporating dedicated modules that are specific for OOD, such as rotation equivariant (RE) backbones. Still, they rely on detectors based on OBB regression heads, which leads to parametrization ambiguities that might generate boundary discontinuity or encoder ambiguity problems, as noted in~\cite{xiao2024theoretically}. 

\noindent
\textbf{Oriented Object Representation and Regression Loss functions:}
A critical component of an OOD architecture is an appropriate OBB representation and its regression (or localization) loss. The earliest and easiest approach is to represent an OBB by parameters $(x,y,w,h, \theta)$ and regress them using a per-parameter $\ell_1$ loss~\cite{yang2019l1ood}, where $(x,y)$ represent the center, $(w,h)$ the dimensions, and $\theta$ the orientation. However, the angular component can produce large $\ell_1$ loss values for similar OBBs due to the discontinuity in the OBB parametrization (recall Figure~\ref{fig:disc:parametrizations}). As mitigation strategies, joint optimization using  IoU-based loss functions have been proposed, such as rotated-IoU (rIoU)~\cite{zhou2019iou}, Pixels IoU (PIoU)~\cite{Chen:EECV:2020} or convex-IoU~\cite{guo2021beyond}. However, they might face differentiability or implementation issues~\cite{yang2021rethinking}. Another set of methods converts OBBs to 2D Gaussian distributions, and explores distribution-based regression loss functions, such as Gauss Wasserstein Distance (GWD)~\cite{yang2021rethinking}, Kullback-Leibler Divergence (KLD)~\cite{yang2021learning}, Bhattacharyya Distance (BD)~\cite{yang2023detecting}, or Probabilistic Intersection-over-Union (ProbIoU) loss~\cite{murrugarra2024probabilistic}. Gaussian-based methods involve simple-to-compute and differentiable regression loss functions, but suffer from the \textit{decoding ambiguity} for square-like objects, for which the angular information cannot be retrieved. Furthermore, they can still suffer from angular discontinuity, as recently mentioned in works that explicitly handle the boundary discontinuity problem, such as~\cite{yu2024boundary,xu2024rethinking}. On the other hand, they provide a natural solution for the \textit{encoding ambiguity} problem for circular objects, unlike loss functions based directly on OBBs. Recent solutions that focus on the boundary discontinuity problem~\cite{yang2020arbitrary,yang2021dense,yang2022scrdet++,yu2023phase,yu2024boundary,xu2024rethinking,xiao2024theoretically} have shown promising results, but are still affected by the encoding ambiguity problem for circular objects since the explore OBBs and enforce angular consistency.

This paper presents a novel regression head for oriented object detection that directly produces the parameters of a 2D Gaussian distribution (i.e., the mean vector and covariance matrix) called GauCho, which can be coupled to any Gaussian-based loss function. To avoid a constrained optimization imposed by the structure of covariance matrices (they need to be positive-definite), we rely on the Cholesky decomposition. As we show in this paper, there is a continuous one-to-one mapping between the GauCho head and the parameters of a Gaussian, which naturally mitigates the boundary discontinuity problem. We also propose to use Oriented Ellipses (OEs) instead of OBBs as the final output of a GauCho-based detector. As shown in Section~\ref{sec:discussion}, OEs are suitable representations for oriented objects typically present in aerial imagery applications and are a natural choice when Gaussian-based loss functions are used. Although GauCho still suffers from the \textit{decoding ambiguity} for square-like objects, it fully solves the \textit{encoding ambiguity} for circular objects.

\section{The Proposed Representation: GauCho}

In this section, we first revise how to obtain Gaussian distributions from OBBs. Then we present the theoretical foundations for GauCho, and how we can adapt detection paradigms with GauCho heads.

\subsection{OBBs and Gaussian Distributions}

Let us consider an OBB with center $(x,y)$, dimensions $(w,h)$ and orientation $\theta \in [-90^\circ, 90^\circ)$ w.r.t. dimension $w$. The mean $\bm{\mu}$  and covariance matrix $C$ of a 2D Gaussian can be obtained through
\begin{equation}
\label{eq:gaussian}
\bm{\mu} = (x,y)^T, ~~ C =  \begin{bmatrix}
a & c \\ c & b
\end{bmatrix} = R\Lambda R^T, 
\end{equation}
\begin{equation}
\label{eq:rotation:diagonal}
R = \begin{bmatrix}
\cos\theta & -\sin\theta \\ \sin\theta & ~\cos\theta
\end{bmatrix}, ~~ 
\Lambda = \begin{bmatrix}
\lambda_w & 0 \\ 0 & \lambda_h
\end{bmatrix} =
s\begin{bmatrix}
w^2 & 0 \\ 0 & h^2
\end{bmatrix},
\end{equation}
which are the rotation and eigenvalue matrices, respectively, and $s$ is a scaling factor that relates the binary OBB with the fuzzy Gaussian representation. For example, $s=1/4$ is selected in~\cite{yang2021learning,yang2023detecting} and $s=1/12$ in~\cite{murrugarra2024probabilistic}. 
We can express the covariance matrix based on the rotation angles and eigenvalues $\lambda_w$ and $\lambda_h$ as
\begin{equation}
\label{eq:cov:rot:eig}
C = \begin{bmatrix}
\lambda_w\cos^2\theta + \lambda_h\sin^2\theta  & \frac{1}{2}(\lambda_w - \lambda_h)\sin(2\theta) \\ \frac{1}{2}(\lambda_w - \lambda_h)\sin(2\theta) & \lambda_w\sin^2\theta + \lambda_h\cos^2\theta
\end{bmatrix},
\end{equation}
which is used by all OOD methods that explore Gaussian loss functions~\cite{yang2021rethinking,yang2021learning,yang2023detecting,yang2023kfiou,murrugarra2024probabilistic}, as it directly relates the shape parameters $(w,h,\theta)$ of an OBB that are regressed by the detectors -- note that $\lambda_w=sw^2$ and $\lambda_h=sh^2$.

However, the mapping from 
$(w,h,\theta)$ to the covariance parameters $(a,b,c)$ is not bijective, as noted in~\cite{murrugarra2024probabilistic,xiao2024theoretically}. When $h=w$, the generated Gaussian is isotropic and $c=0$ for any value of $\theta$. Hence, the OBB cannot be decoded from the Gaussian in these cases. Furthermore, as recently noted in~\cite{yu2024boundary,xu2024rethinking}, Gaussian-based loss functions can also suffer from angular discontinuity at inference time for $w\neq h$. We believe the problem is caused by the OBB-to-Gaussian mapping, particularly for angles close to $\pm 90^\circ$. Considering $C$ in Eq.~\eqref{eq:cov:rot:eig} as a function of the angle $\theta$, we note that $\lim_{\theta \rightarrow 90^\circ}{C(\theta)} = \lim_{\theta \rightarrow -90^\circ}{C(\theta)}$. Hence, an OBB with angle $\theta \approx \pm 90^\circ$ generates a covariance matrix very similar to its counterpart with angle $-\theta$, providing low regression loss values for these two angles using \textit{any} Gaussian-based loss. This behavior leads to two local minima with very different angles, which can impact the training process. Please see the supplementary material for more details.

On the other hand, Eq.~\eqref{eq:cov:rot:eig} shows that all elements in $C$ are continuous and $180^\circ$-periodic functions w.r.t. the orientation $\theta$. 
Hence, the values for $(a,b,c)$ \textit{do not suffer} from the boundary discontinuity problem (recall Figure~\ref{fig:disc:gaussian}), and a possible alternative would be to regress them directly from the network instead of using OBB parameters. However, $(a,b,c)$ are not independent since $C$ must be a positive-definite matrix. For instance, Sylvester's criterion~\cite{gilbert1991positive} states that a Hermitian matrix is positive-definite if and only if all the leading principal minors are positive, which would lead to a constrained optimization. Instead, we explore the Cholesky decomposition of positive-definite matrices, as explained next.

\subsection{The Cholesky Decomposition}

The Cholesky decomposition~\cite{johnson1985matrix} for a positive-definite matrix $C$ provides a unique lower-triangular matrix 
\begin{equation}
L = \begin{bmatrix}
\label{eq:chol}
\alpha & 0 \\ \gamma & \beta
\end{bmatrix}
\end{equation}
with $\alpha, \beta > 0, \gamma \in \mathbb{R},$ such that 
\begin{equation}
\label{eq:cholesky:product}
C = LL^T = \begin{bmatrix}
\alpha^2 & \alpha\gamma \\ \alpha\gamma & \beta^2+\gamma^2
\end{bmatrix} = \begin{bmatrix}
a & c \\ c & b
\end{bmatrix}. 
\end{equation}

The Cholesky parameters $(\alpha, \beta, \gamma)$ provide a unique mapping to a Gaussian, and a deep network can directly regress them as an alternative to the OBB shape parameters $(w,h,\theta)$. Next, we show how to design regression heads based on Gaussian-Cholesky (GauCho) representations.

\subsection{GauCho Regression Head}

Here, we first present bounds that relate GauCho parameters $(\alpha, \beta, \gamma)$ and OBB parameters $(w,h,\theta)$. Then, we show how to adapt anchor-based and anchor-free detectors with GauCho regression heads.

\subsubsection{Bounds on the Matrix Coefficients}

Let us consider the definitions of the covariance matrix and the Cholesky decomposition given from Eq.~\eqref{eq:gaussian} to Eq.~\eqref{eq:cholesky:product}. Also, let us define $\lambda_{min} = \min \{\lambda_h, \lambda_w \}$ and  $\lambda_{max} = \max \{\lambda_h, \lambda_w \}$.

\begin{proposition}[Bounds on the elements of the covariance matrix] 
\label{prop:cov}
The elements $a,b,c$ of the covariance matrix are bounded by the following values. $\lambda_{min} \leq a, b \leq  \lambda_{max}$,
and $|c| \leq  \frac{1}{2}(\lambda_{max} - \lambda_{min})$.
\end{proposition}
\begin{proof}
For the diagonal elements, we show the results for $a$, the proof for $b$ is analogous. From Eq.~\eqref{eq:cov:rot:eig}, we have
\begin{equation}
a \leq \lambda_{max}(\cos^2\theta +\sin^2\theta ) = \lambda_{max},
\end{equation}
\begin{equation}
a \geq \lambda_{min}(\cos^2\theta +\sin^2\theta ) = \lambda_{min}.
\end{equation}
For the off-diagonal element $c$ in Eq.~\eqref{eq:cov:rot:eig}, we have that
\begin{equation}
|c| = \frac{1}{2}|\lambda_w - \lambda_h| |\sin(2\theta)| \leq \frac{1}{2}(\lambda_{max} - \lambda_{min}).
\end{equation}

\end{proof}

\begin{proposition}[Bounds on the elements of the Cholesky matrix] 
\label{prop:chol}The elements $\alpha,\beta,\gamma$ of the Choleky matrix $L$ are bounded by the following values. $\sqrt{\lambda_{min}} \leq \alpha, \beta \leq \sqrt{\lambda_{max}}$, and $|\gamma| < 
\sqrt{\lambda_{max}} - \sqrt{\lambda_{min}}$
\end{proposition}
\begin{proof}

From $\alpha^2 = a$ in Eq.~\eqref{eq:cholesky:product} and Proposition~\ref{prop:cov}, we directly conclude that  $\sqrt{\lambda_{min}} \leq \alpha \leq \sqrt{\lambda_{max}}$. From $C=LL^T$ and the eigendecompositon of $C$ in Eqs.~\eqref{eq:gaussian} and~\eqref{eq:rotation:diagonal}, we have that 
\begin{equation}
    \lambda_w\lambda_h = \det \Lambda =  \det C  = \det LL^T = (\alpha \beta)^2,
\end{equation}
so that $\beta = \sqrt{\lambda_w\lambda_h}/\alpha $ . Hence, we also have that $\sqrt{\lambda_{min}} \leq \beta \leq \sqrt{\lambda_{max}}$.

The proof for the bound on $\gamma$ is provided in the supplementary material.
\end{proof}

Note that $\sqrt{\lambda_{min}} = \sqrt{s}\min\{w,h\}$ and $\sqrt{\lambda_{max}} = \sqrt{s}\max\{w,h\}$, so that all parameters in the Cholesky decomposition present a direct relationship with OBB dimensions. In particular, $\alpha, \beta, |\gamma| \leq \sqrt{s}\max\{w,h\}$. 
Such relationships can be explored by anchor-free and anchor-based detectors.

\subsubsection{Anchor-free heads for GauCho regression}

GauCho can be directly used in anchor-free detectors by regressing directly the desired parameters $(x,y,\alpha,\beta,\gamma)$. Our formulation is based on the popular FCOS (Fully Convolutional One-Stage) object detector, originally designed for HBB detection \cite{tian2019fcos} and extended for OOD by adding an angular component (e.g.~\cite{yang2023detecting,yu2023phase,yu2024boundary}). In the axis-aligned version of FCOS~\cite{tian2019fcos}, the core idea is to regress HBB offsets (top, left, bottom, and right) from a central point  $(p_x, p_y)$ in the feature map based on the stride $t$ of the feature map, which provides an idea of the \textit{scale} of the object. For GauCho, we propose to regress offsets for the center based on the spatial location $(p_x, p_y)$ of the feature map and the corresponding cumulative stride $t$ as
\begin{equation}
       x = p_x + td_x,~y = p_y + td_y, 
\end{equation}
where $d_x,d_y$ have linear activation. For the Cholesky parameters, which relate to the shape of the object, we propose multiplicative offsets given by
\begin{equation}
\alpha = t e^{d_\alpha},~\beta = t e^{d_\beta},~\gamma = t d_\gamma,
\end{equation}
where $d_\alpha,d_\beta,d_\gamma$ are the shape parameters regressed by the GauCho head with linear activation. Note that $d_\alpha=d_\beta=d_\gamma=0$ relates to an axis-aligned object (no rotation) with dimensions proportional to the stride $t$, which is the underlying idea behind~\cite{tian2019fcos}. 

\subsubsection{Anchor-based heads for GauCho regression}

We start from axis-aligned anchors characterized by $(a_x, a_y, a_w, a_h)$, where $(a_x,a_y)$ is the centroid of an anchor with width $a_w$ and height $a_h$. Similarly to~\cite{yang2023detecting}, we regress center offsets $(d_x,d_y)$ with linear activation such that the center of the Gaussian distribution is given by
\begin{equation}
    x = x_a + a_wd_x, ~~y = y_a + a_hd_y,
\end{equation}
which is similar to the HBB formulation of anchors.

Unlike formulations based on OBB heads that regress the OBB dimensions from the anchors and the angle separately, the GauCho shape parameters $(\alpha, \beta, \gamma)$ are tightly coupled. Based on the bounds in Proposition~\ref{prop:chol}, we propose to regress multiplicative offsets $(d_\alpha, d_\beta, d_\gamma)$ with linear activation such that
\begin{equation}
\label{eq:anchor:gaucho}
\alpha = \sqrt{s}a_we^{d_\alpha},\beta = \sqrt{s}a_he^{d_\beta},\gamma = \sqrt{s} \max\{\delta, |a_w-a_h|\}d_\gamma,
\end{equation}
where $s$ is the OBB-to-Gaussian scaling parameter in Eq.~\eqref{eq:rotation:diagonal}. 
For square anchors, $a_w = a_h$ and hence $\lambda_{max} = \lambda_{min}$, which yields $\gamma = 0$ according to Proposition~\ref{prop:chol}. However, the anchors are only a rough estimate of the objects, and using such a rigid constraint on $\gamma$ would prevent rotations when adjusting the dimensions of the anchor. To remedy this problem, we introduced a value $\delta$ in the regression of $\gamma$ set to $\sqrt{\lambda_{min}}$ as the default value. The motivation comes from the common choice of anchor-based detectors of using anchors with aspect ratios 1:1, 1:2, 2:1 (such as Oriented RetinaNet in~\cite{yu2024boundary}). For non-square anchors, we have that $\sqrt{\lambda_{max}} = 2\sqrt{\lambda_{min}}$ so that $\sqrt{\lambda_{max}}  - \sqrt{\lambda_{min}} = \sqrt{\lambda_{min}}$. Hence, selecting $\delta = \sqrt{\lambda_{min}}$ allows a GT annotation associated with a square anchor to stretch to 1:2 or 2:1 aspect ratios. Note that the original horizontal anchor is obtained when $d_\alpha = d_\beta = d_\gamma = 0$.

Some anchor-based OBB detectors work with oriented anchors either in the Region Proposal Network~\cite{xie2024orcnn_beyond} or in a refinement stage~\cite{ding2019roitransformer,Xue2021r3det}. An OBB anchor with shape parameters $(a_w, a_h,\theta)$ relates to a GauCho anchor with shape parameters $( a_{\alpha},  a_{\beta},  a_{\gamma})$ using Eqs.~\eqref{eq:cov:rot:eig}-\eqref{eq:chol}. The proposed refinement of these anchors is given by
\begin{equation}
\alpha = a_{\alpha}e^{d'_\alpha},~\beta = a_{\beta}e^{d'_\beta},~\gamma = a_{\gamma} + \sqrt{s} \max\{\delta, |a_w-a_h|\}d'_\gamma,
\end{equation}
where $(d'_\alpha, d'_\beta, d'_\gamma)$ are the multiplicative offsets regressed by the network with linear activation. Note that $d'_\alpha= d'_\beta= d'_\gamma=0$ maintains the anchor unchanged.

\subsection{Decoding GauCho}

We present two alternatives for decoding GauCho, both of which are suitable for representing oriented objects. The first is the classical OBB representation, whereas the second is based on oriented ellipses (OEs).

\noindent
\textbf{OBB decoding:}
To obtain an OBB from the Gaussian parameters, we follow the same protocol adopted by all other approaches that explore Gaussian loss functions~\cite{yang2021rethinking,yang2021learning,yang2023detecting,yang2023kfiou,murrugarra2024probabilistic}. The mean vector $\bm{\mu}$ in Eq.~\eqref{eq:gaussian} maps directly to the OBB centroid. To obtain the shape parameters, we first retrieve the rotation matrix $R$ and the diagonal matrix $\Lambda$ in Eq.~\eqref{eq:gaussian} by computing the eigenvalues $\lambda_{max} \geq \lambda_{min}$ and eigenvectors of the covariance matrix $C$. The angle $\theta$ is obtained from the orientation of the first eigenvector (which yields a LE parametrization), and the OBB dimensions are decoded from the eigenvalues $\lambda_w$, $\lambda_h$ based on Eq.~\eqref{eq:rotation:diagonal}, i.e., $ w = \sqrt{\lambda_{max}/s}$ and $ h = \sqrt{\lambda_{min}/s}$. This process is well-defined when $\lambda_{max} > \lambda_{min}$, but it generates an angular \textit{decoding ambiguity} for isotropic Gaussians~\cite{xiao2024theoretically,xu2024rethinking}. They are represented by diagonal covariance matrices from which the angular information cannot be retrieved (any pair of vectors forming an orthonormal basis in $\mathbb{R}^2$ are eigenvectors).

\noindent
\textbf{OE decoding:}
Similarly to~\cite{murrugarra2024probabilistic}, we decode an OE from a Gaussian distribution. This is a natural choice since the level sets of Gaussian PDFs are elliptical regions, and there is a \textit{one-to-one} mapping from the space of covariance matrices to OEs. The orientation $\theta$ of the OE is the same as the orientation of the OBB described above. The semi-axes $r_1$ and $r_2$ are defined such that they match the half-sizes of the corresponding OBB, and hence we have $r_1 =  \frac{1}{2}\sqrt{\lambda_{max}/s}$ and $r_2 =  \frac{1}{2}\sqrt{\lambda_{min}/s}$. Note that an isotropic Gaussian relates to a circle, which intrinsically does not present an orientation.

\section{Experimental Results}
\label{sec:experimental}

We adapted different OOD approaches to accommodate the GauCho head, namely: FCOS~\cite{tian2019fcos}, an anchor-free one-stage detector; RetinaNet~\cite{lin2017retinanet}, an anchor-based one-stage detector; R$^3$Det~\cite{Xue2021r3det}, an anchor-based one-stage detector with a refinement step; and RoI-Transformer~\cite{ding2019roitransformer}, an anchor-based two-stage detector. For one-stage detectors, we used Adaptive Training Sample
Selection (ATSS)~\cite{zhang2020bridging} for defining positive and negative training samples, which has been shown to improve the results in OOD~\cite{yang2023detecting}. We used a ResNet-50 (R-50) backbone~\cite{he2016deep} as default for all detectors unless explicitly mentioned. For all detectors, we generated results using different Gaussian-based loss functions: GWD~\cite{yang2021learning}, KLD~\cite{yang2023detecting} and \piou{}~\cite{murrugarra2024probabilistic}. To ensure a fair comparison, we used implementations in the MMRotate benchmark~\cite{zhou2022mmrotate} with the default configuration files for each detector, which contain hyperparameter settings (learning rate, number of epochs, augmentation policy, etc.) for the datasets DOTA~\cite{Xia:2018:CVPR} and HRSC~\cite{liu2017high}. These parameters were defined for the baseline detectors based on OBB heads, and we used the same parameters for the GauCho head (we believe better results can be achieved by fine-tuning these parameters, which is out of the scope of this paper). We also explored the UCAS-AOD dataset~\cite{zhu2015orientation}, which contains several square OBBs related to planes and is used as an example of decoding ambiguity~\cite{xu2024rethinking}. In all experiments with HRSC and UCAS-AOD, we used a TITAN Xp GPU with 12GB of VRAM. For the experiments with DOTA v1.0 and v1.5, we used an A100 GPU with 80GB of VRAM. More details about the datasets and training protocol are provided next.




\noindent
\textbf{DOTA~\cite{Xia:2018:CVPR,ding_dota1.5:2021}} contains images collected from Google Earth by GF-2 and JL-1 satellites provided by the China Centre for Resources Satellite Data and Application. Aerial images were supplemented with imagery from CycloMedia B.V. DOTA v1.0~\cite{Xia:2018:CVPR} and DOTA v1.5~\cite{ding_dota1.5:2021} are annotated with the same images, but DOTA v1.5 provides revised and updated annotations, including tiny objects. In both scenarios, there are 1,869 images in the training set and 937 in the test set. We run each experiment for 12 epochs using random flip augmentation by a 50\% chance. 


\noindent
\textbf{HRSC 2016~\cite{liu2017high}} contains images gathered from Google Earth with ship annotations. It has 1,070 images in total: 626 for training and 444 for testing. We run each experiment for 72 epochs using random vertical, horizontal, and diagonal flips at a 25\% chance each and random rotation at a 50\% chance.

\noindent
\textbf{UCAS-AOD~\cite{zhu2015orientation}} it is a remote sensing dataset with two categories: cars and planes. It comprises 1,510 annotated images, divided into 1,110 for training and 400 for testing. Since there are no default configuration files for UCAS-AOD in MMRotate, we used the same ones as HRSC.



\begin{table*}[ht!]
    \centering
    \footnotesize
    \caption{Comparison of different detectors using the original OBB head and the proposed GauCho with single-scale training/testing. AP values are computed with OBB or OE representations, and the best result for each detector-loss is shown in bold.  }
    \resizebox{.95\linewidth}{!} {
    \begin{tabular}{l|l|ccc|ccc|ccc}
        \hline
        
      \multirow{2}{*}{Detector}  & \multirow{2}{*}{Head-Loss} & AP$_{50}$ & AP$_{75}$ & AP & AP$_{50}$ & AP$_{75}$ & AP & AP$_{50}$ & AP$_{75}$ & AP \\ \cline{3-11} 
        
        & & \multicolumn{3}{c|}{HRSC (OBB)}  & \multicolumn{3}{c|}{UCAS-AOD (OBB/OE)}  & \multicolumn{3}{c}{DOTA v1.0 (OBB)} \\ 
        \hline
        \multirow{6}{*}{FCOS}

        & OBB-GWD & 88.93 & \textbf{76.67} & 84.93 & \textbf{90.22}/\textbf{90.26} & \textbf{55.75}/\textbf{65.42} & \textbf{53.73}/ \textbf{59.52} & 69.76 & 34.68 & 37.89 \\
        & \textbf{GauCho-GWD} & \textbf{89.76} & 76.30 & \textbf{85.26} & 90.17/90.17 & 53.84/64.84 & 52.33/58.55 & \textbf{71.22} & \textbf{35.85} & \textbf{38.63} \\ \cline{2-11}
        & OBB-KLD & 88.38 & 66.42 & 82.24 & \textbf{90.22}/\textbf{90.26} & 50.03/64.96 & 52.48/59.04  & 71.74 & 28.30 & 36.18 \\
        & \textbf{GauCho-KLD} & \textbf{89.94} & \textbf{78.99} & \textbf{87.86} & 90.04/90.07 & \textbf{55.01}/\textbf{65.06} & \textbf{52.72}/\textbf{59.37} & \textbf{72.16} & \textbf{33.27} & \textbf{38.46} \\ \cline{2-11}
        & OBB-\piou{}  & \textbf{90.08} & 76.84 & 87.27 & \textbf{90.17}/90.16 & 46.73/64.83 & 52.27/\textbf{59.27} & 71.31 & 37.34 & 39.80 \\
        & \textbf{GauCho-\piou{}} & 89.86 & \textbf{78.21} & \textbf{87.58} & 90.14/\textbf{90.18} & \textbf{55.35}/\textbf{65.27} & \textbf{53.03}/59.08 & \textbf{72.86} & \textbf{37.69} & \textbf{40.65} \\
        \hline \hline
         \multirow{6}{*}{RetinaNet-ATSS} 
        & OBB-GWD & 89.47 & 75.65 & 83.83 & 89.72/89.83 & 34.37/60.16 & 46.28/56.08  & \textbf{71.51} & 36.34 & 39.59 \\
        & \textbf{GauCho-GWD} & \textbf{90.32} & \textbf{78.34} & \textbf{86.39} & \textbf{89.79}/\textbf{89.83} & \textbf{50.40}/\textbf{62.69} & \textbf{51.55}/\textbf{57.92} & 71.36 & \textbf{38.00} & \textbf{40.29} \\ \cline{2-11}
        & OBB-KLD & 90.17 & 77.62 & 86.00 & 89.64/89.65 & 49.33/62.98  & 50.73/\textbf{57.10} &  72.05 & 37.72 & 40.47 \\
        & \textbf{GauCho-KLD} & \textbf{90.40} & \textbf{80.45} & \textbf{88.56} & \textbf{89.71}/\textbf{89.71} & \textbf{50.18}/\textbf{63.01} & \textbf{50.84}/57.08 & \textbf{72.71} & \textbf{38.47} & \textbf{40.57} \\ \cline{2-11}
        & OBB-\piou{} & 90.20 & 77.67 & 87.37 & \textbf{89.87}/\textbf{89.87} & 48.93/\textbf{63.16} & \textbf{51.03}/57.09 & 72.14 & \textbf{39.77} & \textbf{40.97}\\
        & \textbf{GauCho-\piou{}} & \textbf{90.48} & \textbf{80.35} & \textbf{88.56} & 89.78/89.74  & \textbf{50.61}/63.04 & 51.34/\textbf{57.43} & \textbf{73.21} & 37.63 & 40.91 \\
        \hline \hline
        \multirow{6}{*}{R$^3$Det-ATSS} 
        & OBB-GWD & \textbf{89.66} & 65.68 & \textbf{81.90} & \textbf{90.02}/\textbf{90.07} & 38.60/61.40 & 47.54/56.68  & 67.98 & 34.89 & 37.11 \\
        & \textbf{GauCho-GWD} & 89.52 & \textbf{65.83} & 81.77 & 89.94/89.95 & \textbf{49.87}/\textbf{62.15} & \textbf{51.41}/\textbf{56.72} & \textbf{70.53} & \textbf{35.74} & \textbf{39.07} \\ \cline{2-11}
        & OBB-KLD & \textbf{89.92} & 53.46 & 79.32 & \textbf{89.96}/\textbf{90.00} & \textbf{52.05}/\textbf{63.87} & \textbf{52.07}/\textbf{57.35} & 70.77 & \textbf{36.98} & \textbf{38.90} \\
        & \textbf{GauCho-KLD} & 89.65 & \textbf{62.66} & \textbf{82.97} & 89.90/89.93  & 49.79/63.65 & 51.48/57.11 & \textbf{70.83} & 33.48 & 37.65 \\ \cline{2-11}
        & OBB-\piou{} & 89.19 & 51.37 & 78.40 & \textbf{89.98}/\textbf{90.19} & 44.85/\textbf{64.28} & 50.23/\textbf{57.67} & 70.85 & \textbf{36.66} & \textbf{38.91} \\
        & \textbf{GauCho-\piou{}} & \textbf{90.02} & \textbf{76.43} & \textbf{85.76} & 89.95/89.96 & \textbf{51.72}/63.95 & \textbf{52.01}/57.41 & \textbf{71.23} & 33.64 & 37.89 \\
        \hline \hline
        \multirow{6}{*}{RoI Transformer} 
        & OBB-GWD & 90.35 & \textbf{88.51} & \textbf{80.40} & \textbf{90.31}/\textbf{90.32} & 58.37/69.07 & \textbf{55.20}/59.54 & 75.38 & \textbf{42.53} & \textbf{42.87} \\
        & \textbf{GauCho-GWD} & \textbf{90.35} & 59.28 & 79.72 & 90.28/90.31 & \textbf{58.53}/\textbf{69.47} & 54.84/\textbf{59.54} & \textbf{75.66} & 41.05 & 42.38 \\ \cline{2-11}
        & OBB-KLD & \textbf{90.52} & \textbf{89.36} & \textbf{90.25} & \textbf{90.35}/\textbf{90.35} & \textbf{64.15}/\textbf{73.71} & \textbf{57.42}/61.32 & \textbf{76.55} & \textbf{47.54} & \textbf{45.96}  \\
        & \textbf{GauCho-KLD} & 90.50 & 88.80 & 90.12 & 90.32/90.34 & 56.90/70.34 & 54.60/\textbf{61.40} & 76.35 & 43.79 & 44.32  \\ \cline{2-11}
        & OBB-\piou{} & 90.54 & 89.12 & 90.16  & \textbf{90.35}/\textbf{90.37} & \textbf{63.05}/\textbf{73.40} & \textbf{56.76}/60.81 & 75.49 & \textbf{46.31} & \textbf{45.18}  \\
        & \textbf{GauCho-\piou{}} & \textbf{90.58} & \textbf{89.13} & \textbf{90.20}  & 90.32/90.33 & 61.41/70.59 & 55.57/\textbf{60.91} & \textbf{76.09} & 42.60 & 43.90 \\
        \hline
    \end{tabular}
    
    \label{table:datasets_obb_results}
    }
\end{table*}

Table~\ref{table:datasets_obb_results} shows a comparison of OBB and GauCho heads for different detectors and Gaussian-based regression loss functions for the HRSC, UCAS-AOD, and DOTA v1.0 datasets computed with AP$_{50}$, AP$_{75}$ and AP metrics. For HRSC and UCAS-AOD, both heads presented similar metrics for different detectors and loss functions. For DOTA v1.0, GauCho presented consistently better results than the OBB head for FCOS in all metrics; for the remaining detectors, GauCho and OBB heads presented mostly similar results, and the difference can be possibly explained by random seed selection during training. As mentioned before, UCAS-AOD contains several almost-square OBBs related to planes, which leads to decoding ambiguity when using Gaussian loss function (either OBB or GauCho heads). This behavior can be observed in the more restrictive  AP$_{75}$ metric, which presents a strong decrease compared to AP$_{50}$ and more variation across different heads and loss functions. The use of OEs instead of OBBs for representing the final detection results partially mitigates the decoding ambiguity problem: while the angle still cannot be retrieved, orientation variations have a small impact on the decoded OE -- in the limit case, a perfect square is mapped to the same circle regardless its orientation. Table~\ref{table:datasets_obb_results} also reports the AP metrics computed based on the IoU of OE representations instead of OBBs, and we note a considerable increase in the AP$_{75}$ values and more consistency within each detector; on the other hand, the AP$_{75}$ values using OEs are very similar to those using OBBs.

Table~\ref{table:dota1.5_gbb_results} compares OBB and GauCho heads for DOTA v1.5. We show only the results of the anchor-free detector FCOS because this dataset presents several small objects that require adjustments in the anchors of RetinaNet, R$^3$Det and Roi Transformers, which are not provided in MMRotate. Similarly to the results of FCOS in DOTA v1.0, using GauCho instead of OBB heads yields a consistent improvement in the AP$_{50}$ (about 1.1\% on average) for all tested regression loss functions. The per-category AP$_{50}$ also increased with GauCho for most classes in DOTA v1.5.

\begin{table*}[ht!]
    \centering
    \footnotesize
    \caption{Results for OBB detection for DOTA v1.5  (per-class and average AP$_{50}$) with FCOS R-50 and single-scale training}.
    \resizebox{18cm}{!} {
    \begin{tabular}{l|cccccccccccccccc|c}
        
        Head-Loss & PL & BD & BR & GTF & SV & LV & SH & TC & BC & ST & SBF & RA & HA & SP & HC & CC & AP$_{50}$ \\ \hline
    
        OBB-GWD & 71.48 & 72.11 & 45.75 & 53.72 & 57.28 & 73.54 & 80.23 & 90.88 & 76.76 & 73.81& 51.79 & 68.63 & 55.40 & 65.16 & 55.11 & 10.79 & 62.65 \\
        \textbf{GauCho-GWD} & \textbf{78.06} & 71.62 & \textbf{47.01} & \textbf{59.24} & \textbf{60.46} & \textbf{74.08} & \textbf{84.12} & \textbf{90.88} & \textbf{77.02} & 73.52 & \textbf{51.83} & \textbf{69.70} & \textbf{59.84} & \textbf{71.39} & 49.62 & 5.56 & \textbf{64.00} \textbf{\upper{(+1.35)}} \\
        \hline
        OBB-KLD & 78.21 & 75.71 & 48.04 & 55.19 & 59.98 & 73.76 & 84.10 & 90.85 & 76.25 & 74.42 & 56.28 & 69.47 & 61.68 & 69.89 & 50.57 & 7.46 & 64.49 \\ 
        \textbf{GauCho-KLD} & \textbf{78.96} & 72.90 & 47.33 & 54.46 & \textbf{62.20} & \textbf{75.03} & \textbf{85.78} & \textbf{90.85} & 75.82 & 74.34 & 54.12 & \textbf{70.00} & \textbf{63.55} & \textbf{71.57} & \textbf{54.26} & \textbf{16.97} & \textbf{65.51} \textbf{\upper{(+1.02)}}\\
        \hline
        OBB-\piou{} & 78.50 & 73.43 & 45.81 & 57.40 & 57.03 & 73.92 & 80.05 & 90.85 & 75.08 & 74.18 & 52.96 & 69.29 & 60.22 & 69.40 & 55.61 & 14.37 & 64.26 \\
        \textbf{GauCho-\piou{}} & 76.42 & 72.78 & \textbf{48.42} & \textbf{59.72} & \textbf{61.65} & \textbf{75.19} & \textbf{84.83} & \textbf{90.88} & \textbf{76.44} & 73.88 & \textbf{56.75} & \textbf{69.51} & \textbf{62.98} & 67.79 & 50.55 & 13.65 & \textbf{65.09} \textbf{\upper{(+0.83)}}
        \\\hline
        \end{tabular}
    }
    \label{table:dota1.5_gbb_results}
\end{table*}

For a comparison with SOTA OOD approaches, we performed additional experiments with DOTA v1.0 using multiscale (MS) training/testing, which has been shown to improve the AP metrics~\cite{yang2023detecting} significantly. Table~\ref{table:dota_gbb_results} shows the results with GauCho and a subset of competitive approaches reported in the recent paper~\cite{xu2024rethinking} with multiscale training/testing that presented the best AP$_{50}$ values. We note that FCOS-GauCho performs slightly better than CenterNet-ACM, which is also an anchor-free detector. Also,  coupling RoI-Transformer~\cite{ding2019roitransformer} with GauCho yields better results than the same detector with OBB head with ACM loss~\cite{xu2024rethinking}. Note that the ACM loss requires an additional hyperparameter (its weight), unlike GauCho. 
Unfortunately, there are few papers that report results for DOTA v1.5 with multiscale training/testing. For a fair comparison, we report the result of the competitive anchor-free detector DAFNe~\cite{lang2021dafne}, which achieved an mAP of 71.99 using a ResNet101 backbone. In comparison, FCOS-GauCho achieves an mAP of 73.56 with the same backbone and without additional architectural changes introduced by the compared detector.

\begin{table*}[ht!]
    \centering
    \footnotesize
    \caption{SOTA results for DOTA v1.0   (per-class and average AP$_{50}$) with multiscale training/testing. }
    \resizebox{18cm}{!} {
    \begin{tabular}{l|cccccccccccccccc|c}
        Method & PL & BD & BR & GTF & SV & LV & SH & TC & BC & ST & SBF & RA & HA & SP & HC & CC & AP$_{50}$ \\ \hline
        \multicolumn{18}{c}{DOTA v1.0} \\ \hline       
        RoI-Transformer~\cite{ding2019roitransformer} &88.64 &78.52& 43.44 &75.92 &68.81& 73.68& 83.59& 90.74& 77.27& 81.46& 58.39& 53.54& 62.83 &58.93& 47.67&-& 69.56 \\
        DAL~\cite{ming2021dynamic} & 88.61 & 79.69 & 46.27 & 70.37 & 65.89 & 76.10 & 78.53 & 90.84 & 79.98 & 78.41 & 58.71 & 62.02 & 69.23 & 71.32 & 60.65 & - &  71.78 \\
        CFC-Net~\cite{ming2021cfc} &89.08& 80.41 &52.41& 70.02& 76.28& 78.11& 87.21& 90.89& 84.47& 85.64& 60.51& 61.52& 67.82& 68.02& 50.09& - &  73.50 \\
        CSL~\cite{yang2020arbitrary} & \textbf{90.25} & 85.53& 54.64& 75.31& 70.44& 73.51& 77.62 &90.84& 86.15& 86.69 &69.60& 68.04& 73.83& 71.10 &68.93 & - & 76.17\\
        R$^3$Det~\cite{Xue2021r3det} & 89.80 &83.77& 48.11 &66.77 &78.76 &83.27& 87.84 &90.82 &85.38 &85.51 &65.67& 62.68 &67.53& 78.56 &72.62 & - & 76.47\\
        GWD~\cite{yang2021rethinking} &86.96 &83.88& 54.36 &77.53& 74.41& 68.48 &80.34& 86.62& 83.41 &85.55& \textbf{73.47} &67.77& 72.57& 75.76 &73.40& - &  76.30\\
        SCRDet++~\cite{yang2022scrdet++} &90.05 &84.39& 55.44 &73.99& 77.54 &71.11& 86.05 &90.67& 87.32& 87.08 &69.62& 68.90& 73.74& 71.29 &65.08& - &  76.81\\
        KFIoU~\cite{yang2023kfiou} &89.46 &85.72 &54.94& 80.37& 77.16 &69.23& 80.90 &90.79 & \textbf{87.79} &86.13 &73.32& 68.11& 75.23& 71.61& 69.49 & - & 77.35\\
        DCL~\cite{yang2021dense} & 89.26 & 83.60 & 53.54 & 72.76 & 79.04 &82.56 &87.31 &90.67 &86.59 & 86.98 &67.49 &66.88 & 73.29& 70.56 & 69.99 & - & 77.37\\
        RIDet~\cite{ming2021optimization} & 89.31 & 80.77 &54.07 &76.38& 79.81 &81.99& \textbf{89.13} & 90.72& 83.58& 87.22& 64.42& 67.56& 78.08& 79.17 &62.07 & - & 77.62\\
        PSCD~\cite{yu2023phase} & 89.86 & \textbf{86.02} & 54.94 & 62.02 &81.90 & 85.48 &88.39 &90.73 &86.90& \textbf{88.82} & 63.94 & 69.19 &76.84& \textbf{82.75} & 63.24 & - &78.07 \\
        KLD~\cite{yang2021learning} & 88.91 & 85.23 & 53.64 & \textbf{81.23} & 78.20 & 76.99 & 84.58 & 89.50 & 86.84 & 86.38 & 71.69 & 68.06 & 75.95 & 72.23 & 75.42 & - & 78.32 \\
        CenterNet-ACM~\cite{xu2024rethinking} & 89.84 & 85.50 & 53.84 & 74.78 & 80.77 & 82.81 &88.92 & 90.82 & 87.18 & 86.53 & 64.09 & 66.27 & 77.51 & 79.62 & 69.57 & - & 78.53\\
        RoI-Transformer-ACM~\cite{xu2024rethinking} & 85.55 & 80.53 & \textbf{61.21} &75.40 &80.35 &85.60 &88.32 &89.88& 87.13 &87.10 &68.15 & 67.94 & 78.75 & 79.82 & \textbf{75.96} & - & 79.45\\

        \hline

        \textbf{FCOS-GauCho} & 88.96 & 81.01 & 57.39 & 72.21 & \textbf{82.40} & 85.41 & 88.51 & 90.85 & 85.42 & 86.40 & 66.42 & \textbf{70.19} & 76.10 & 80.42 & 71.00 & - &  78.85 \\
        \textbf{GauCho-RoITransformer} & 89.58 & 85.12 & 60.03 & 80.32 & 79.81 & \textbf{85.71} & 88.59 & \textbf{90.90} & 87.70 & 88.23 & 70.51 & 68.68 & \textbf{79.29} & 80.57 & 74.10
 & - & \textbf{80.61} \\
        \end{tabular}
    }
    \label{table:dota_gbb_results}
\end{table*}

\section{Discussion}
\label{sec:discussion}

\noindent
\textbf{A critical analysis of OBBs vs. OEs:} let us consider the categories and OBB annotations in the popular DOTA 1.0 dataset~\cite{Xia:2018:CVPR}.  
Some of these categories are \textit{geometrically} oriented, such as \texttt{ships} (SH), \texttt{large-vehicles} (LV), and \texttt{tennis courts} (TC), among others. These objects present a geometric dominant axis, which defines the orientation of the object. Other categories are \textit{semantically} oriented, such as \texttt{planes} (PL) or \texttt{helicopters} (HC). These objects are often characterized by (near-)square bounding boxes for which the orientation is characterized by the \textit{content} of the objects (\eg, the nose of an airplane). Finally, some categories are \textit{ill-oriented}, such as \texttt{swimming pools} (SP) with irregular shapes, or even \textit{not oriented}, such as \texttt{roundabouts} (RA) or \texttt{storage tanks} (ST) with a circular profile. 

The top row of Figure~\ref{fig:oe:obb:mask} shows examples of the four types of oriented categories mentioned above, represented as OBBs (red) and OEs (green). We also show the corresponding segmentation masks provided in~\cite{waqas2019isaid} on the bottom row. We can observe that the geometrically oriented objects in  Figure~\ref{fig:geometric:oriented} can be well represented by OEs and OBBs; for the semantically oriented objects illustrated in Figure~\ref{fig:semantic:oriented}, it is difficult or even impossible to retrieve the orientation from the OEs due to the \textit{decoding ambiguity}, as opposed to the OBB representation; for ill-oriented objects, as shown in Figure~\ref{fig:no:orientation}, the orientation provided by the OBB is rather arbitrary, while the OE is roughly circular. Finally, OEs are well-suited for circular objects, whereas OBBs provide an \textit{artificial} orientation for an object that does not provide one (the \textit{encoding ambuiguity} problem), as shown in Figure~\ref{fig:circular} and previously illustrated in Figure~\ref{fig:ra:re}. For a quantitative comparison between OBBs and OEs, we computed the IoU between both representations against the segmentation masks in the whole dataset, and analyzed the per-category median values. In nine of the 15 categories, the median IoU values using OEs were higher than OBBs (more details provided in the supplementary material), which corroborates the viability of OEs as an alternative representation for oriented objects.

\begin{figure}[ht!]
    \centering
    \subfloat[][]{\includegraphics[height =  0.15\textwidth]{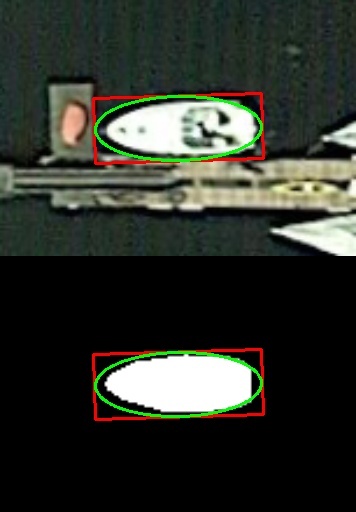}~\includegraphics[height =  0.15\textwidth]{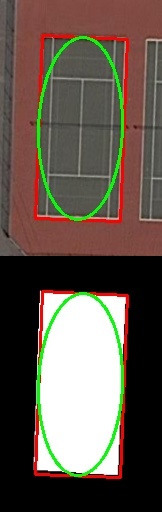} \label{fig:geometric:oriented}}
    \subfloat[][]{\includegraphics[height =  0.15\textwidth]{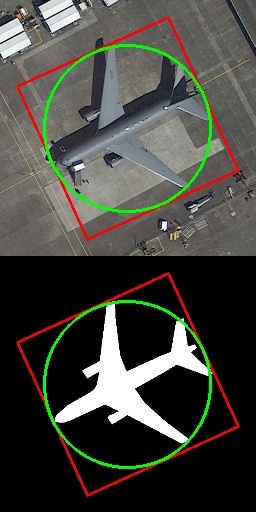}~\includegraphics[height =  0.15\textwidth]{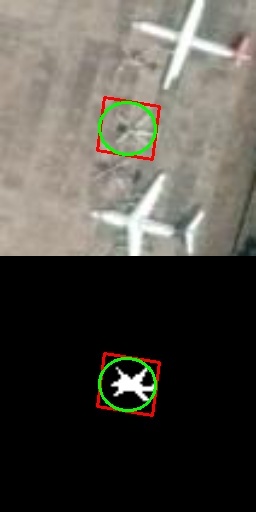} \label{fig:semantic:oriented}}
    \subfloat[][]{\includegraphics[height =  0.15\textwidth]{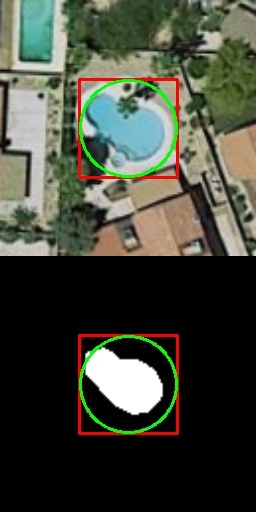}\label{fig:no:orientation}}
    \subfloat[][]{\includegraphics[height =  0.15\textwidth]{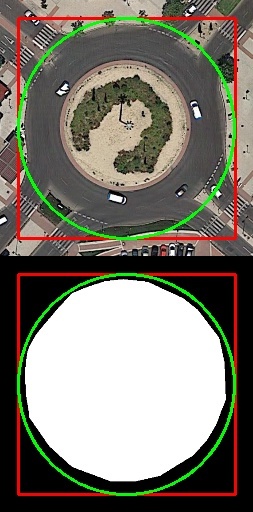}\label{fig:circular}}
    \caption{\label{fig:oe:obb:mask} Examples of object representations using OEs and OBBs for different categories overlaid with the RGB image (top) and annotated segmentation mask (bottom). (a) \textit{Geometrically oriented} objects. (b) \textit{Semantically oriented} objects. (c) \textit{Ill-oriented} objects. (d) \textit{Circular objects}.    }
\end{figure}

\begin{figure}[th!]
    \centering
    \includegraphics[width = 0.5\textwidth]{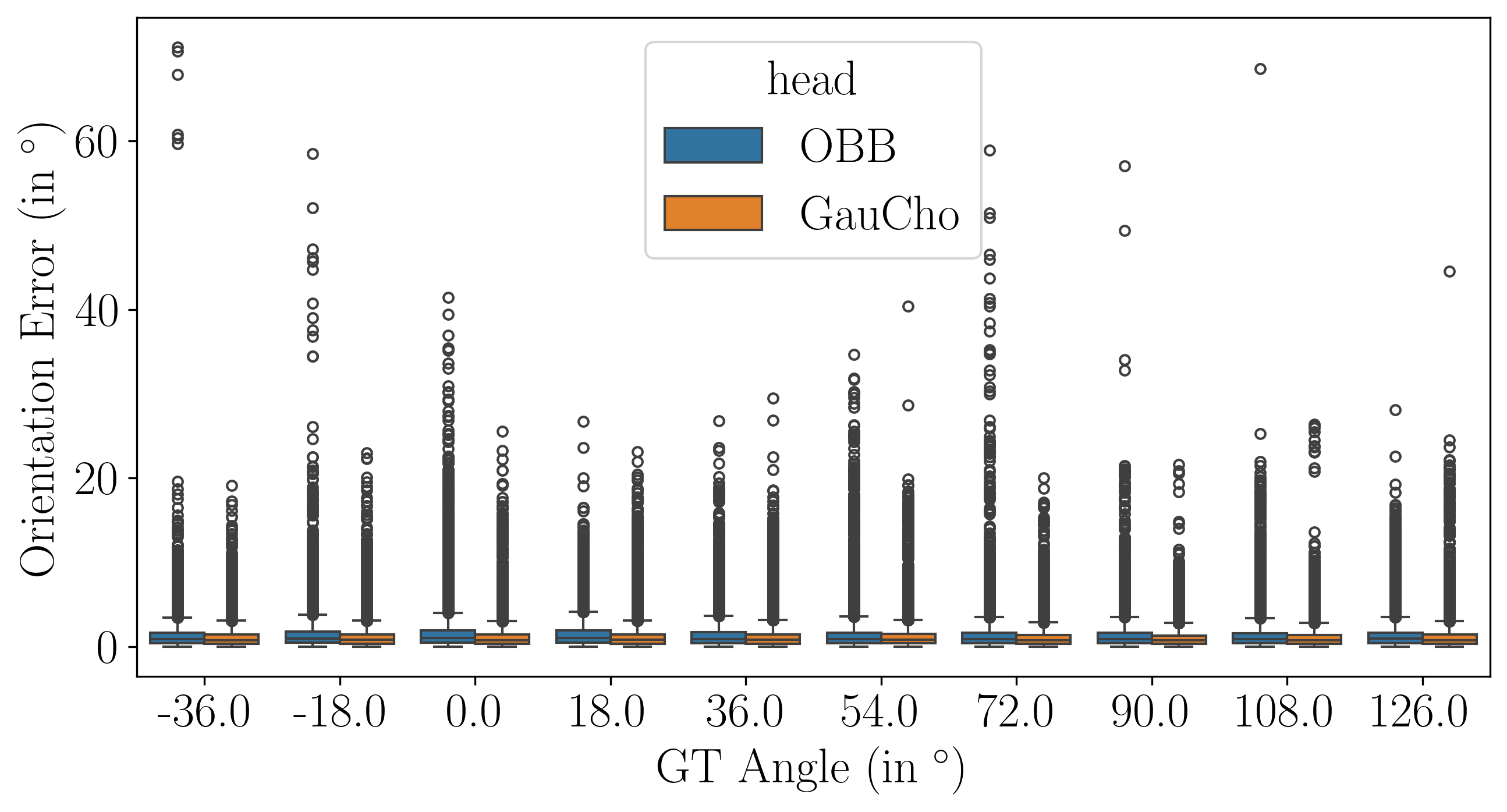}
    \caption{\label{fig:fcos:obb:gaucho:orientation}  Orientation Error for different GT \textit{orientation bins} using FCOS with OBB and GauCho heads in HRSC. }
\end{figure}



\noindent
\textbf{Orientation Consistency:}
 one of the motivations for GauCho was to mitigate the orientation discontinuity problem in OOD, which impacts orientation accuracy. A closely related issue is \textit{rotation equivariance} (RE): if an image undergoes rotation, the predictions must be rotated accordingly. Although some detectors impose RE by construction~\cite{han2021redet,lee2024fred}, RE can be \textit{learned} by using image rotations as augmentation primitives when training a detector~\cite{yang2023detecting} -- in fact, Han et al.~\cite{han2021redet} showed that a non-RE baseline detector can yield better AP values with rotation augmentation than a native RE detector, at the cost of increased training time. However, explicitly imposing angular information for categories that are ill-oriented or not-oriented leads to contradictions. For example, let us consider an instance of \texttt{roundabout} (RA) with the corresponding OBB and OE annotations shown in Figure~\ref{fig:ra:angle0}. The rotated versions of the image and annotation by $15^\circ$ and $45^\circ$ are shown in Figures~\ref{fig:ra:angle15} and~\ref{fig:ra:angle45}, respectively. In this example, imposing a guiding direction with an OBB representation leads to inconsistencies since the network must learn different angular information based on nonexistent visual cues, which is a direct consequence of the \textit{encoding ambiguity} problem. On the other hand, the OE/Gaussian representation is fully compatible since it is not affected by rotations. Not that the \textit{decoding ambiguity} problem for square-like objects with Gaussian representations does not violate the idea of RE. It is also interesting to note that the configuration files for DOTA in MMRotate benchmark present a ``trick'' in the rotation augmentation strategy: for categories RA and ST, the annotation is only rotated by multiples of $90^\circ$.

Similarly to~\cite{li2022orientedrep}, we estimated the orientation consistency of a detector using the Orientation Error. We used the HRSC dataset in our analysis since it contains only ships that are \textit{geometrically} oriented objects. For each image of the test set, we generated synthetically rotation versions from $1^\circ$ to $360^\circ$ in steps of $1^\circ$, which yields a uniform distribution of ship orientations and also implicitly evaluates RE properties. We compared the results of FCOS with the original OBB head and GauCho, both trained with the \piou{} loss. We grouped the orientation of the GT annotations into ten angular bins and generated boxplots with the absolute orientation errors for each head, as shown in Figure~\ref{fig:fcos:obb:gaucho:orientation}. We note that the orientation errors using GauCho are smaller for all orientation bins, with fewer outliers. The Average Orientation Error (AOE) considering all bins is $1.11^\circ$ for GauCho and $1.36^\circ$ for the OBB head. Since the AOE can be affected by outliers, we also computed the Median Orientation Error (MOE) for both heads, resulting in values of $0.79^\circ$ for GauCho and $0.94^\circ$ for the OBB head. In both metrics, GauCho presented smaller angular errors than the baseline OBB head. For the sake of comparison, the mean and median errors for FCOS-PSC~\cite{yu2023phase}, which explicitly deals with angular information, are $1.14^\circ$ and $0.83^\circ$, respectively, which were slightly larger than FCOS-GauCho.




\section{Conclusions}
\label{sec:conclusions}

This paper presented a novel regression head for oriented object detection (OOD) that produces Gaussian distributions as an alternative to the \textit{de facto} OBB head. Instead of regressing directly the covariance matrix, which leads to constrained optimization, we explored the Cholesky decomposition to obtain the Gaussian parameters (GauCho). We showed that GauCho mitigates the \textit{angular discontinuity problem} present in OBB representations and provided theoretical bounds that relate GauCho parameters with actual OBB dimensions, which are explored to devise anchor-free and anchor-based GauCho heads. We also stressed the \textit{encoding ambiguity} of circular objects when using OBB heads and advocated using Oriented Ellipses (OEs) as an alternative to OBBs when using GauCho. 

Our experimental results show that GauCho can be integrated with existing object detection paradigms (anchor-free or anchor-based, one-stage or two-stage), being a viable alternative to the traditional OBB head. GauCho produces similar AP metrics when compared to OBB heads for different detectors, Gaussian-based loss functions, and datasets, while producing smaller angular errors. When using multiscale training/testing strategies, GauCho achieves results comparable to or better than the SOTA on the popular DOTA dataset.

\section{Appendix}
\label{sec:conclusions}

\subsection{Possible cause for orientation discontinuity using Gaussian-based loss functions}

In the paper, we hypothesize that the angular discontinuity problem with Gaussian-based loss functions recently noted in~\cite{yu2024boundary,xu2024rethinking} is caused by the OBB to Gaussian mapping. To illustrate the problem, let us consider an origin-centered ground-truth (GT) OBB with shape parameters $(w,h,\theta) = (3,1,89^\circ)$ with LE encoding. Figure~\ref{fig:discontinuity} shows the plot of the KLD loss $\mathcal{L}_{\text{KLD}}$~\cite{yang2021learning} (the same behavior happens to \textit{any} Gaussian-based loss) as a function of $\theta$ in the range $[-90^\circ, 90^\circ)$. The global minimum is reached for $\theta = 89^\circ$, but a local minimum (almost as low as the global one) is achieved for $\theta = -90^\circ$. In fact, the corresponding OBBs are geometrically very similar, as shown on the left of Figure~\ref{fig:discontinuity}. The loss function is clearly not convex, and the network might not learn the angular information properly in this scenario.
\begin{figure}[ht!]
    \centering
\includegraphics[width = 0.5\textwidth]{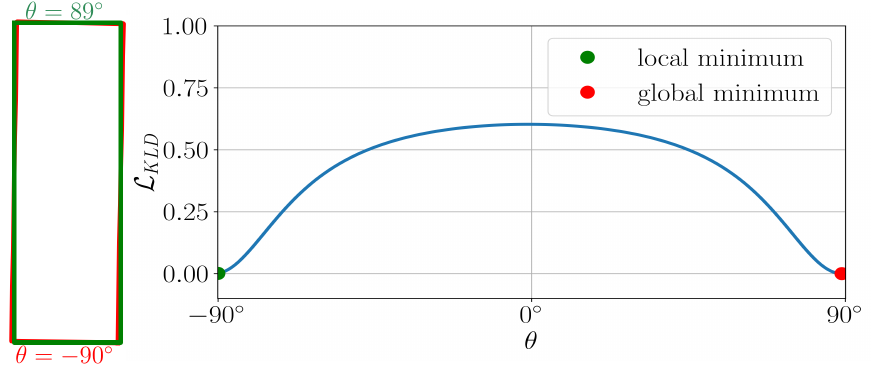}
    \caption{\label{fig:discontinuity}
    When regressing angular information from a Gaussian-based loss, the global angular minimum (red, $\theta = 89^\circ$) might be close to the discontinuous counterpart (green, $\theta  = -90^\circ$).
}
\end{figure}

\begin{figure*}[t!]
    \centering
\includegraphics[width = \textwidth]{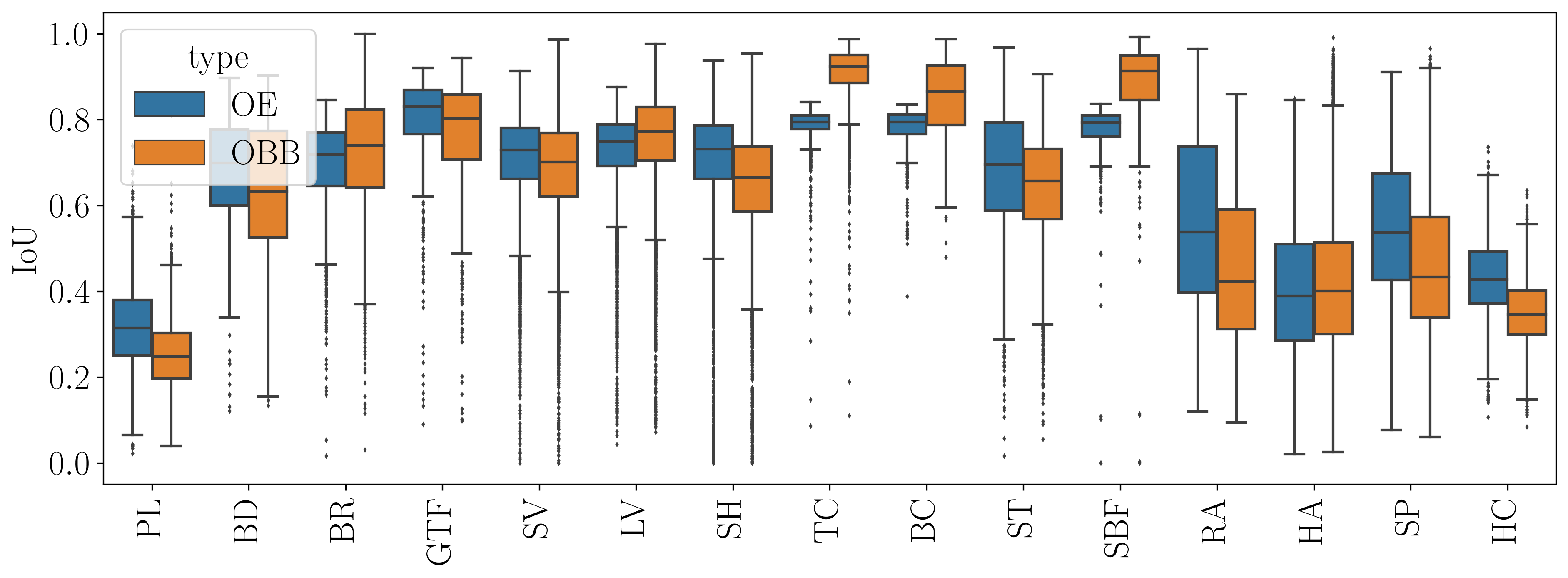}   \caption{\label{fig:oe:obb} IoU between OBB and OE representations with the corresponding segmentation masks in DOTA.}
\end{figure*}

\subsection{Proof of Proposition 3.2}

Here, we show the proof for the bound on the off-diagonal element of the Cholesky matrix. First, we revise the notation and the re-state the proposition.

Let us consider a covariance matrix expressed as a function of the eigenvalues $\lambda_w$, $\lambda_h$ and the orientation $\theta$: 
\begin{equation}
\label{eq:cov:rot:eig}
C = \begin{bmatrix}
\lambda_w\cos^2\theta + \lambda_h\sin^2\theta  & \frac{1}{2}(\lambda_w - \lambda_h)\sin(2\theta) \\ \frac{1}{2}(\lambda_w - \lambda_h)\sin(2\theta) & \lambda_w\sin^2\theta + \lambda_h\cos^2\theta
\end{bmatrix}.
\end{equation}

Also, let us recall the Cholesky decomposition characterized by a lower-triangular matrix $L$
\begin{equation}
L = \begin{bmatrix}
\label{eq:chol}
\alpha & 0 \\ \gamma & \beta
\end{bmatrix}
\end{equation}
with $\alpha, \beta > 0, \gamma \in \mathbb{R},$ such that $C = LL^T$, i.e.,
\begin{equation}
\label{eq:cholesky:product}
C = \begin{bmatrix}
\alpha^2 & \alpha\gamma \\ \alpha\gamma & \beta^2+\gamma^2
\end{bmatrix} = \begin{bmatrix}
a & c \\ c & b
\end{bmatrix}. 
\end{equation}

\noindent
\textbf{Proposition:} $|\gamma| \leq 
\sqrt{\lambda_{max}} - \sqrt{\lambda_{min}}$

\begin{proof}
From Eqs.~\eqref{eq:cov:rot:eig} and~\eqref{eq:cholesky:product}, we have that
\begin{equation}
\label{eq:alpha**2}
\alpha^2 = a = \lambda_w\cos^2\theta + \lambda_h\sin^2\theta 
\end{equation}
Since
\begin{align}
\lambda_w =& \frac{\lambda_w + \lambda_h}{2} + \frac{\lambda_w - \lambda_h}{2},\\
\lambda_h =& \frac{\lambda_w + \lambda_h}{2} + \frac{\lambda_h - \lambda_w}{2},
\end{align}
we can rewrite Eq.~\eqref{eq:alpha**2} as
\begin{align}
\alpha^2 &= \frac{\lambda_w + \lambda_h}{2} (\cos^2\theta + \sin^2\theta)  \\
&+ \frac{\lambda_w - \lambda_h}{2} (\cos^2\theta - \sin^2\theta) \\
&= \frac{\lambda_w + \lambda_h}{2} + \frac{\lambda_w - \lambda_h}{2}\cos(2\theta)
\end{align}

From Eqs.~\eqref{eq:cov:rot:eig} and~\eqref{eq:cholesky:product}, we have that
\begin{align}
\gamma^2 = \frac{c^2}{\alpha^2} 
= \frac{1}{2}\frac{(\lambda_w - \lambda_h)^2\sin^2(2\theta) }{(\lambda_w + \lambda_h) + (\lambda_w - \lambda_h)\cos(2\theta)}.
\end{align}

Defining $x = \cos(2\theta)$, we have that $x\in[-1,1]$. We can express $\gamma^2$ as a function of $x$, given by
\begin{equation}
\gamma^2 = f(x) =\frac{1}{2}\frac{(\lambda_w - \lambda_h)^2(1-x^2) }{(\lambda_w + \lambda_h) + (\lambda_w - \lambda_h)x},
\end{equation}
so that 
$$f'(x) = 
\frac{\left(\lambda_{h} - \lambda_{w}\right)^{2} \left(\lambda_{h} x^{2} - 2 \lambda_{h} x + \lambda_{h} - \lambda_{w} x^{2} - 2 \lambda_{w} x - \lambda_{w}\right)}{2 \left(\lambda_{h} x - \lambda_{h} - \lambda_{w} x - \lambda_{w}\right)^{2}}.
$$

The only solution of $f'(x) = 0$ in the interval $[-1,1]$ is given by
\begin{equation}
x^*  =\frac{ \lambda_{h} + \lambda_{w} - 2 \sqrt{\lambda_{h}} \sqrt{\lambda_{w}}}{\lambda_{h} - \lambda_{w}}.
\end{equation}

Since $f(-1) = f(1) = 0$ and $\gamma^2$ is non-negative, the global maximum occurs at $x = x^*$. The maximum value of $\gamma^2$ is given by
\begin{equation}
f(x^*) = \lambda_{h} + \lambda_{w}  - 2 \sqrt{\lambda_{h}} \sqrt{\lambda_{w}}= (\sqrt{\lambda_{w}} - \sqrt{\lambda_{h}})^2.
\end{equation}

Finally, we have that
\begin{equation}
\max |\gamma| = \sqrt{\max{\gamma^2}} =  \sqrt{\lambda_{max}} - \sqrt{\lambda_{min}},
\end{equation}
\end{proof}
\subsection{Comparison Between OBBs and OEs in DOTA}

To show that Oriented Ellipses (OEs) can be used as an alternative to Oriented Bounding Boxes (OEs) for representing typical objects in oriented object detection, we performed a study based on the DOTA 1.0 dataset~\cite{Xia:2018:CVPR}, which provides OBB annotations. From each OBB, we generated an OE representation with the same orientation of the OBB and semi-axis composed of half of the OBB dimensions, as explained in the paper. For each annotation, represented as both OBB and OE, we computed the IoU with the segmentation masks provided in~\cite{waqas2019isaid}. 

Figure~\ref{fig:oe:obb} shows the IoU values for OE and OBBs considering all 15 categories of the DOTA dataset. The median IoU value computed with OEs is higher than the IoU using OBBs in nine of the sixteen categories: PL, BD, GTF, SV, SH, ST, RA, SP, and HC. In particular, we highlight the relatively low IoU values for RA (\texttt{roundabout}) using OEs. The main cause is the discrepancy between the OBB and segmentation mask annotations, as 
illustrated in Figure~\ref{fig:iou:ra}. In Figures~\ref{fig:ra:1} and~\ref{fig:ra:2}, the mask comprises only the roundabout, but the OBB also includes the surrounding street. In Figure~\ref{fig:ra:3}, both OBB and mask comprise the roundabout and street, while in Figure~\ref{fig:ra:4}, they comprise only the roundabout.


\begin{figure}[ht!]
    \centering
    \subfloat[][]{\includegraphics[height = 0.25\textwidth]{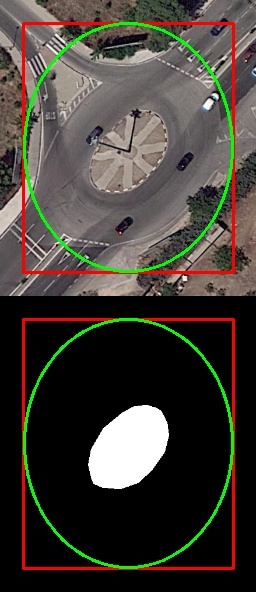}\label{fig:ra:1}}
    \subfloat[][]{\includegraphics[height = 0.25\textwidth]{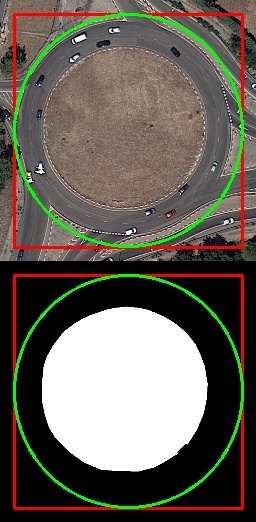}\label{fig:ra:2}}
    \subfloat[][]{\includegraphics[height = 0.25\textwidth]{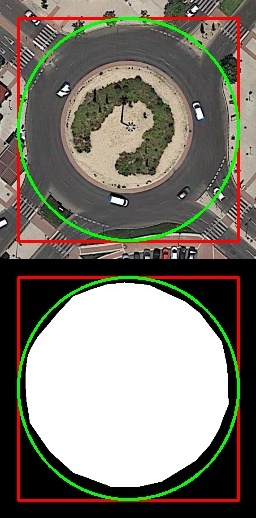}\label{fig:ra:3}}
    \subfloat[][]{\includegraphics[height = 0.25\textwidth]{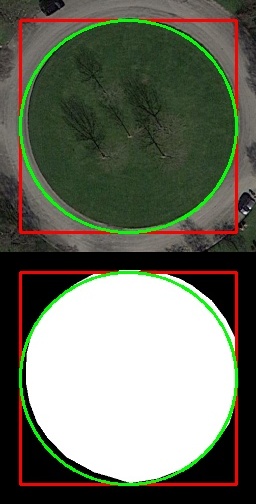}\label{fig:ra:4}}
    \caption{\label{fig:iou:ra} Examples of inconsistencies between OBB and segmentation masks for \texttt{roundabouts} in DOTA and iSAID.
    }
\end{figure}

\subsection{Visual Results}
Here, we show some results of oriented object detection using GauCho. We selected some representative images of the tested datasets (DOTA, HRSC and UCAS-AOD) and showed the results as both OBBs and OEs.

Figure~\ref{fig:aoe_hrsc} shows a visual comparison of FCOS-GauCho and FCOS-Baseline over rotated images of the HRSC dataset aiming to evaluate the rotation equivariance (RE) assumption. Although the results of both detectors are mostly coherent, using the GauCho head yields better orientation consistency (see the maximum orientation error for each detector).

\begin{figure}[ht!]
    \centering
    \subfloat[][]{\includegraphics[height = 0.18\textwidth]{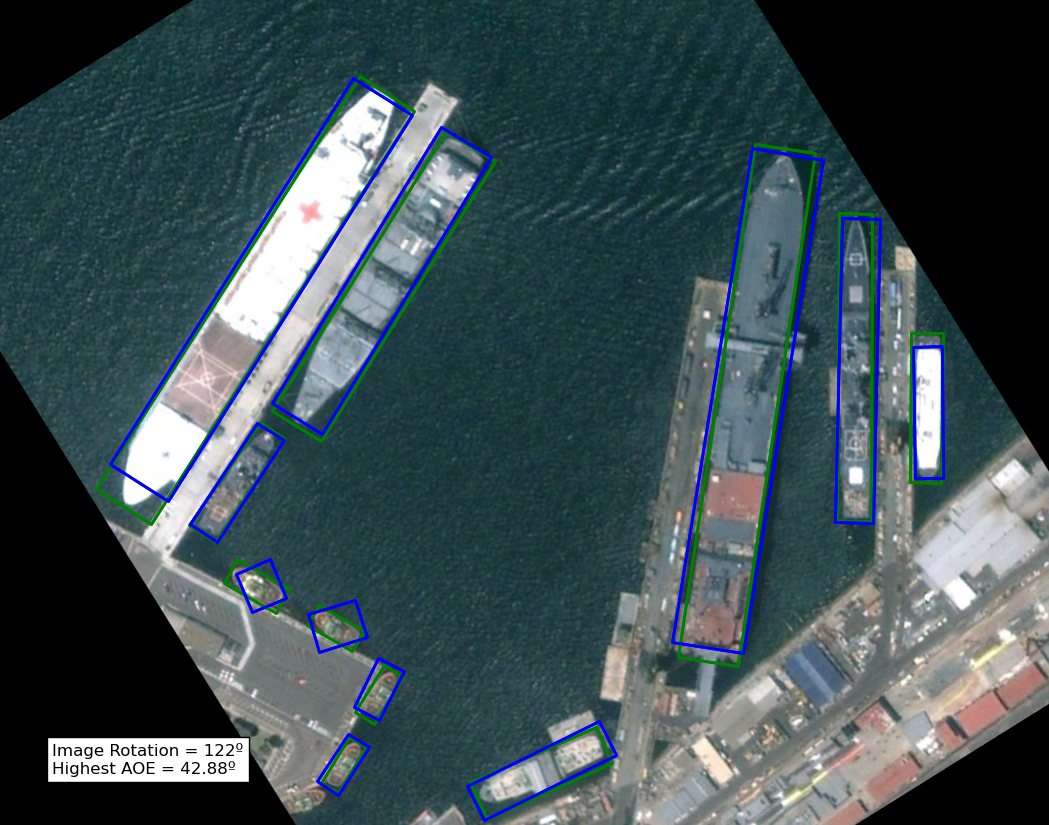}\label{fig:aoe_hrsc_cmp1:1}}~
    \subfloat[][]{\includegraphics[height = 0.18\textwidth]{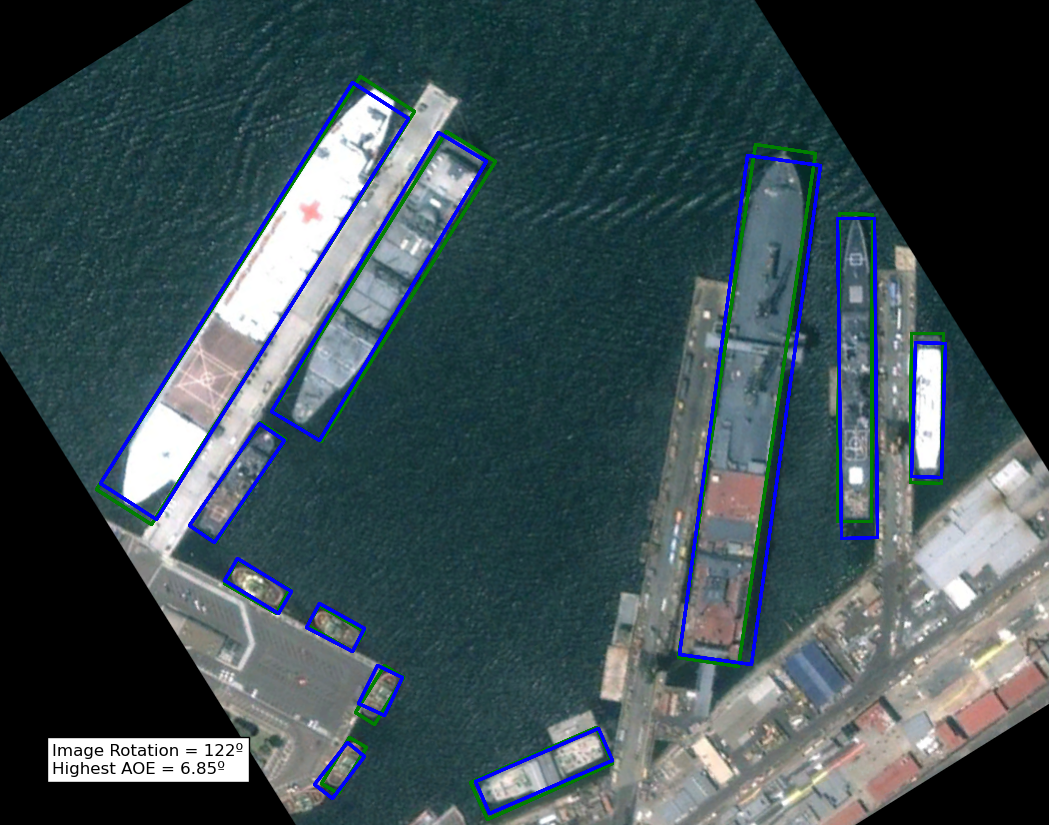}\label{fig:aoe_hrsc_cmp1:2}}
    \hfill
    \subfloat[][]{\includegraphics[height = 0.18\textwidth]{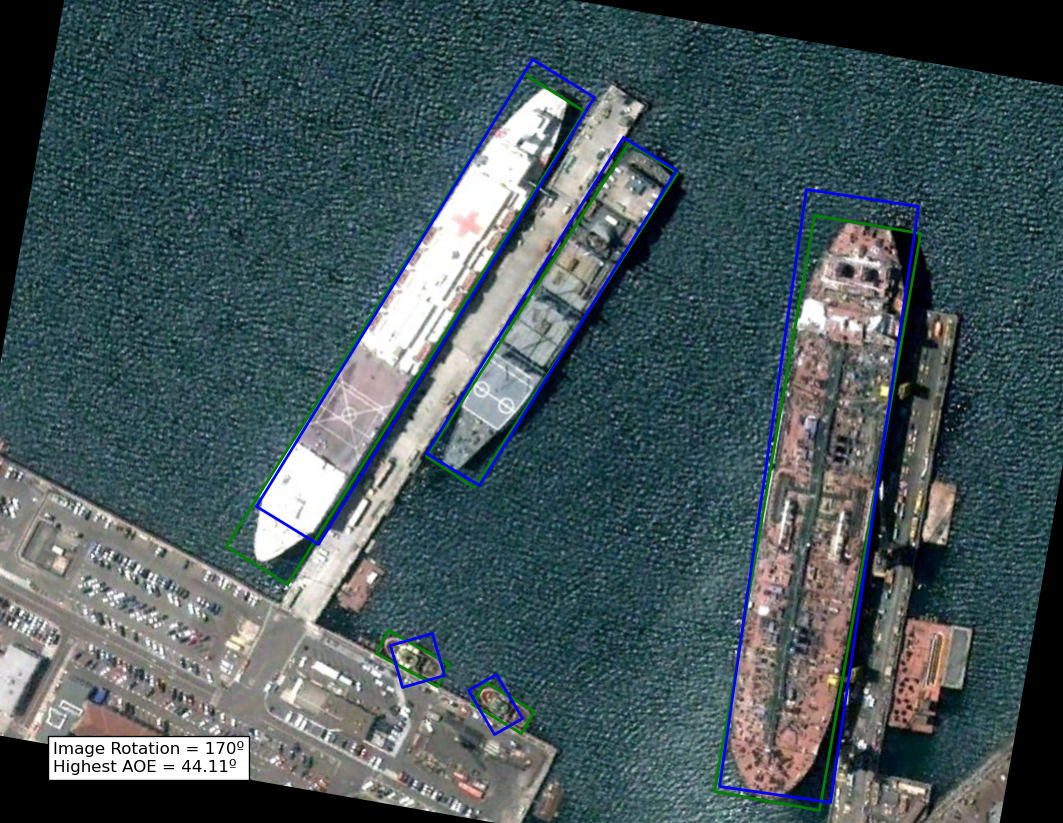}\label{fig:aoe_hrsc_cmp2:1}}~
    \subfloat[][]{\includegraphics[height = 0.18\textwidth]{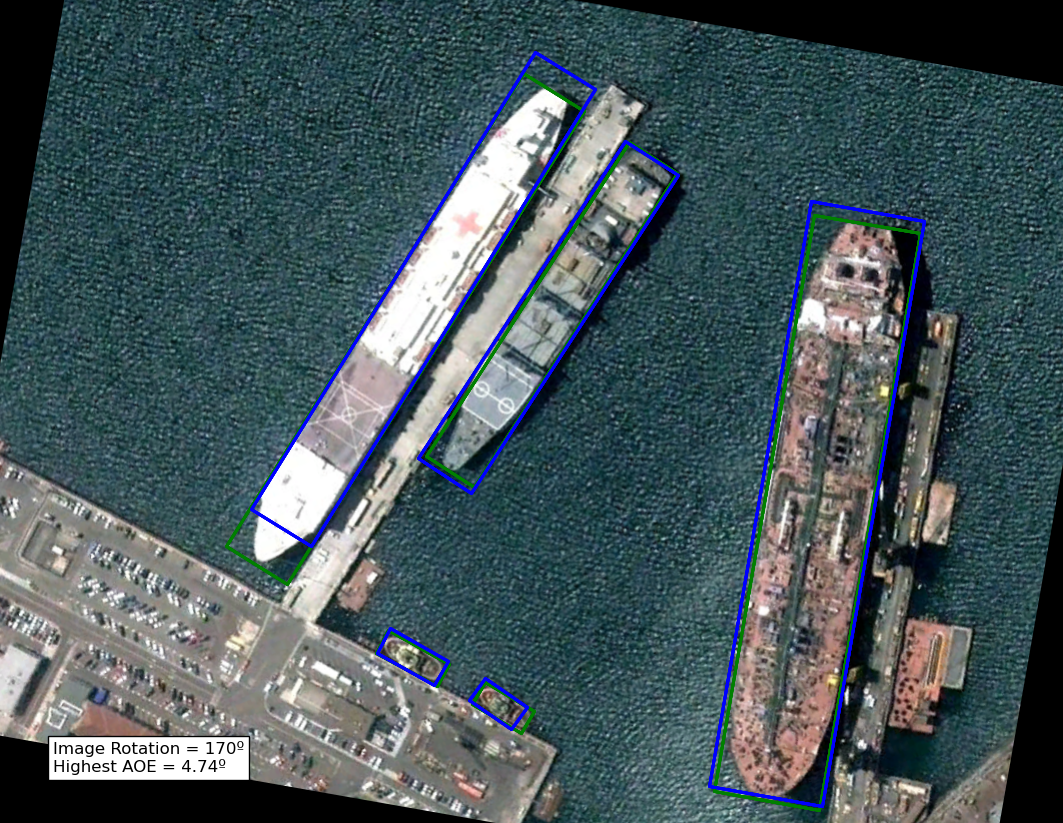}\label{fig:aoe_hrsc_cmp2:2}}
    \caption{\label{fig:aoe_hrsc}
    Visual comparison of FCOS-Baseline (a, c) and FCOS-GauCho (b, d) under rotated images of the HRSC dataset. Blue boxes represent detections that matched a ground truth box with IoU $\geq0.5$ and green boxes represent the matched ground truth boxes. Highest AOE refers to the highest achieved absolute orientation error of a detection in relation to its matched ground truth.
}
\end{figure}


Figure~\ref{fig:ucas_ellipse} shows detection results using FCOS-GauCho with ProbIoU loss for UCAS-AOD using both OBB and OE representations. This particular image shows the known \textit{decoding ambiguity} problem for square-like objects when Gaussian-based loss functions are used: the orientation of the planes cannot be retrieved using OBB representations, leading to larger discrepancies between predictions and GT annotations (Figure~\ref{fig:ucas_cmp1:1}). On the other hand, the proposed OEs are fully compatible with GauCho and Gaussian-based loss functions, since the orientation has a small impact on ellipses with small aspect ratios, as shown in Figure~\ref{fig:ucas_cmp1:2}. Figure~\ref{fig:visual:results_dota} shows a similar result for the DOTA dataset: Figures~\ref{fig:visual:results_dota}a-d illustrate detection results as OBBs, while Figures~\ref{fig:visual:results_dota}e-h depict the same detections as OEs. Unfortunately, we do not have access to GT annotations of the test set in DOTA. However, we note that FCOS-GauCho produces coherent results.
\begin{figure}[ht!]
    \centering
    \subfloat[][]{\includegraphics[width = 0.5\textwidth]{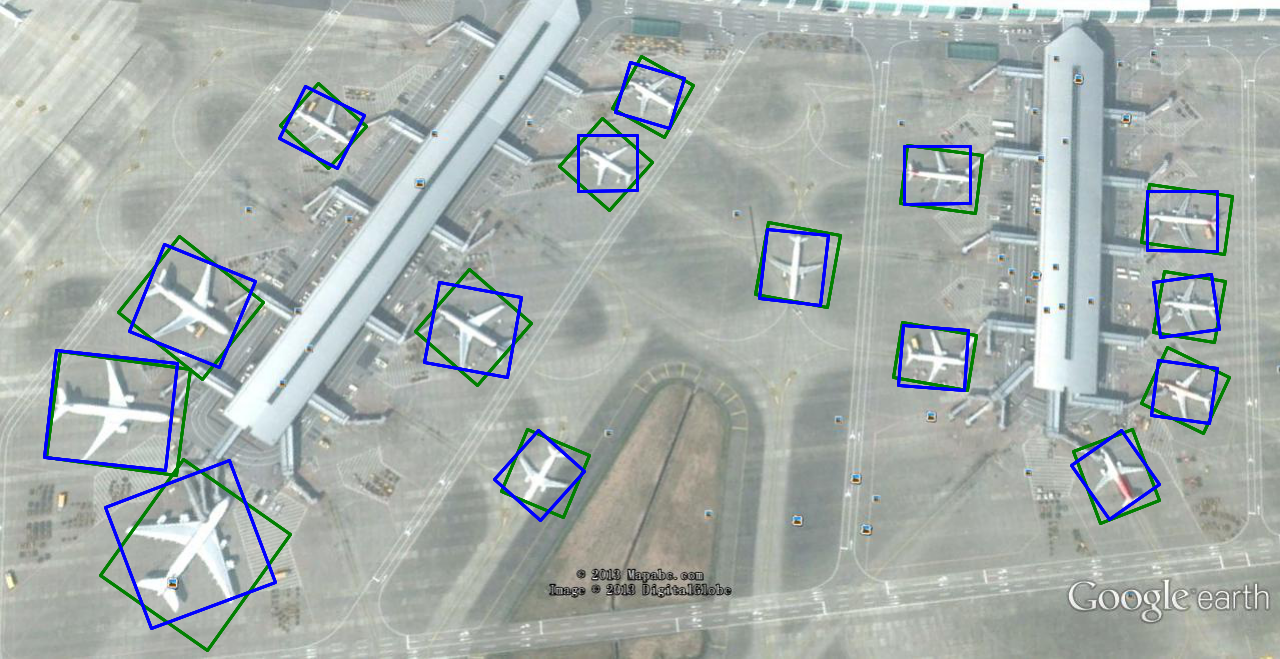}\label{fig:ucas_cmp1:1}} \\
    \subfloat[][]{\includegraphics[width = 0.5\textwidth]{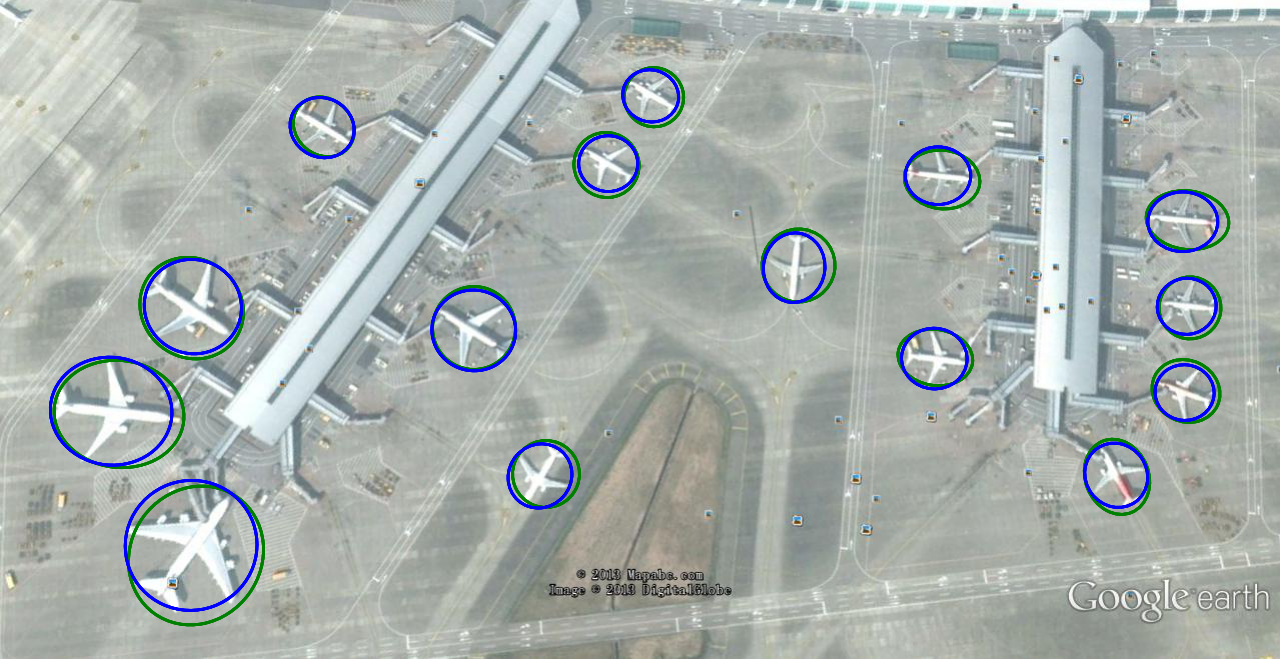}\label{fig:ucas_cmp1:2}}
    \caption{\label{fig:ucas_ellipse}
Detection outputs (blue) for FCOS-GauCho on the UCAS-AOD dataset and GT annotations (green), using (a) OBB and (b) OE representations. 
}

\end{figure}


\begin{figure*}[htb]
    \centering
    \subfloat[]{
    \includegraphics[height = 4.1cm, width = 4.1cm]{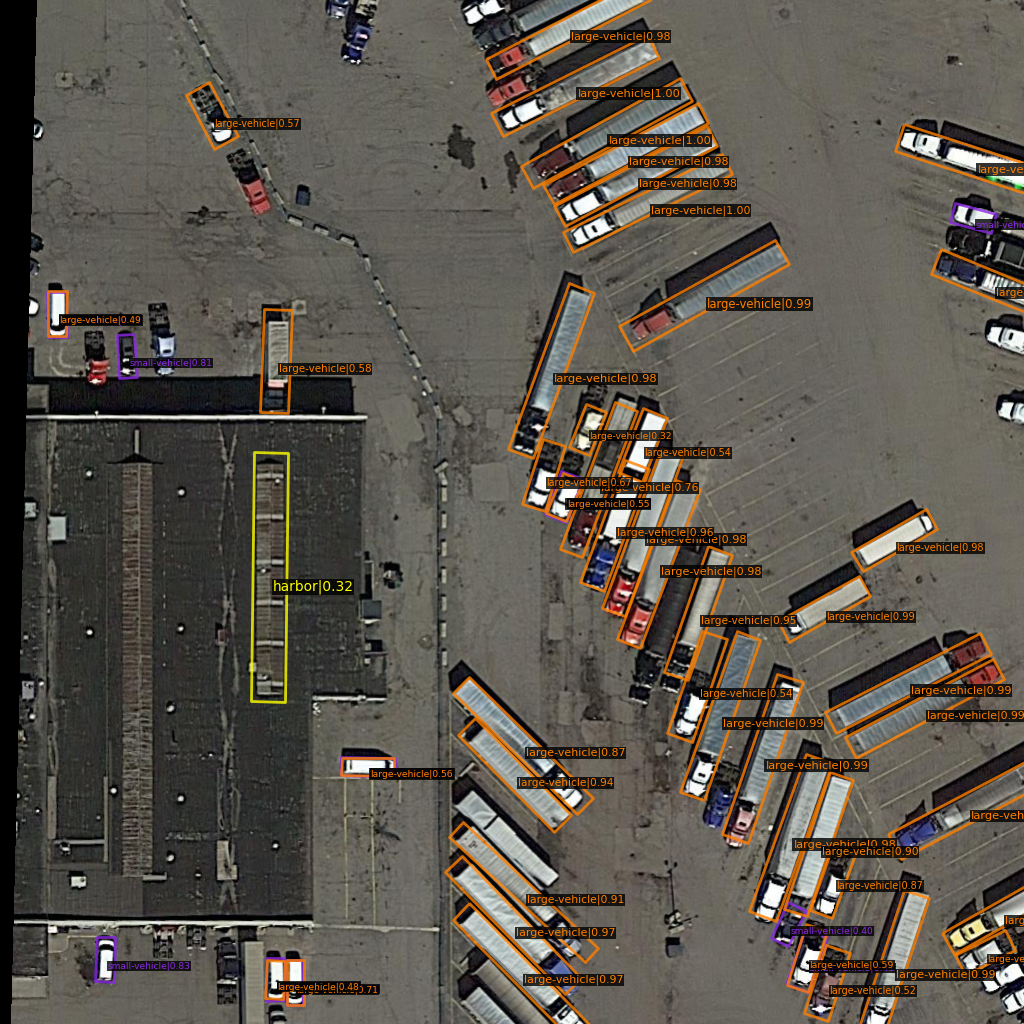}
    \label{fig:dota_gaucho_rbboxes1}}
    \subfloat[]{
    \includegraphics[height = 4.1cm, width = 4.1cm]{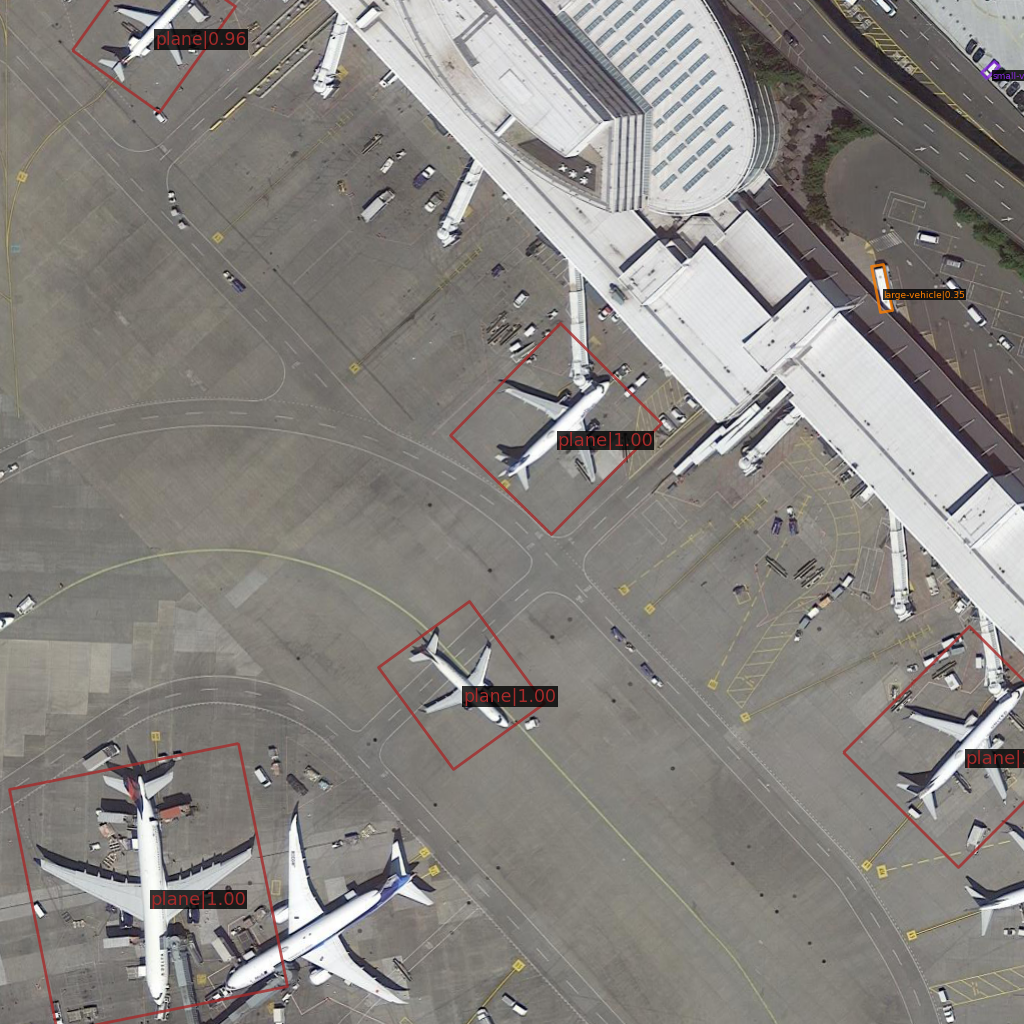}
    \label{fig:dota_gaucho_rbboxes2}}
    \subfloat[]{
    \includegraphics[height = 4.1cm, width = 4.1cm]{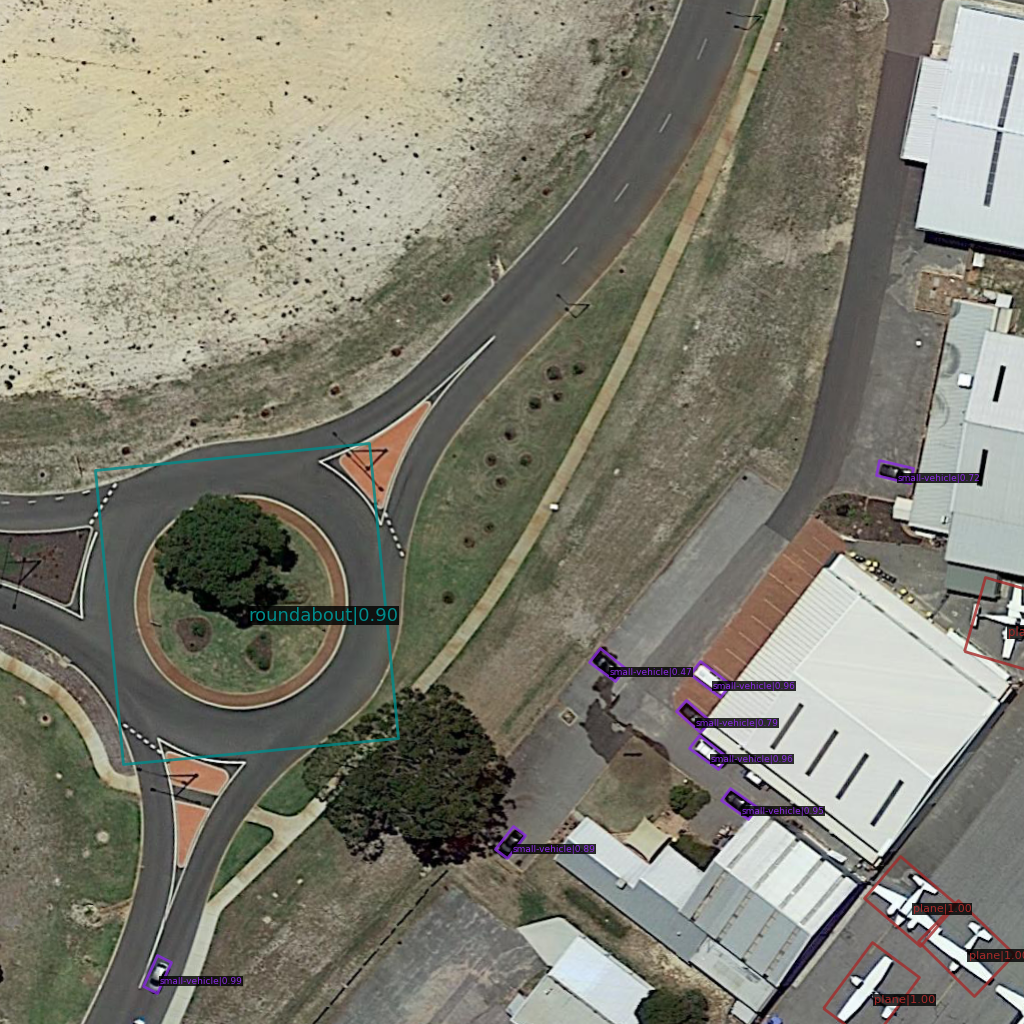}
    \label{fig:dota_gaucho_rbboxes3}}
    \subfloat[]{
    \includegraphics[height = 4.1cm, width = 4.1cm]{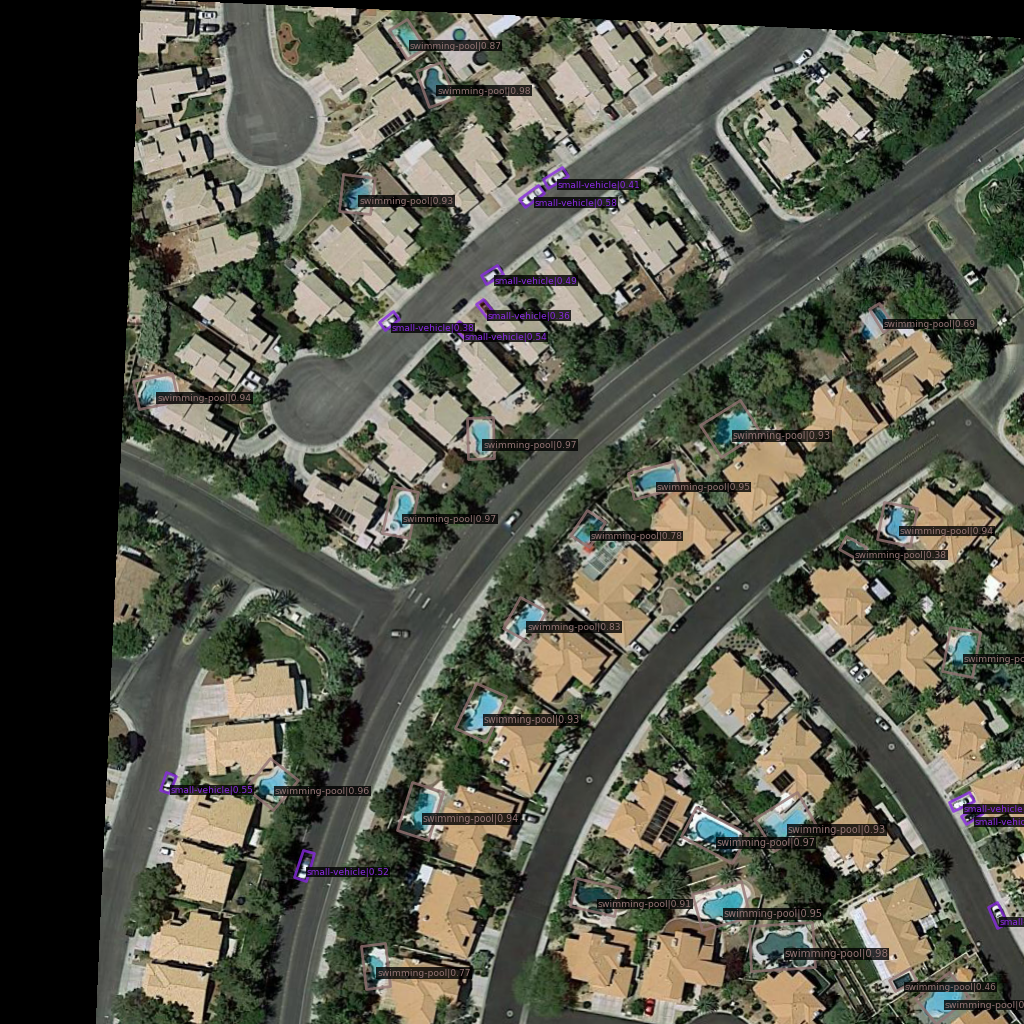}} \\
    \subfloat[]{
    \includegraphics[height = 4.1cm, width = 4.1cm]{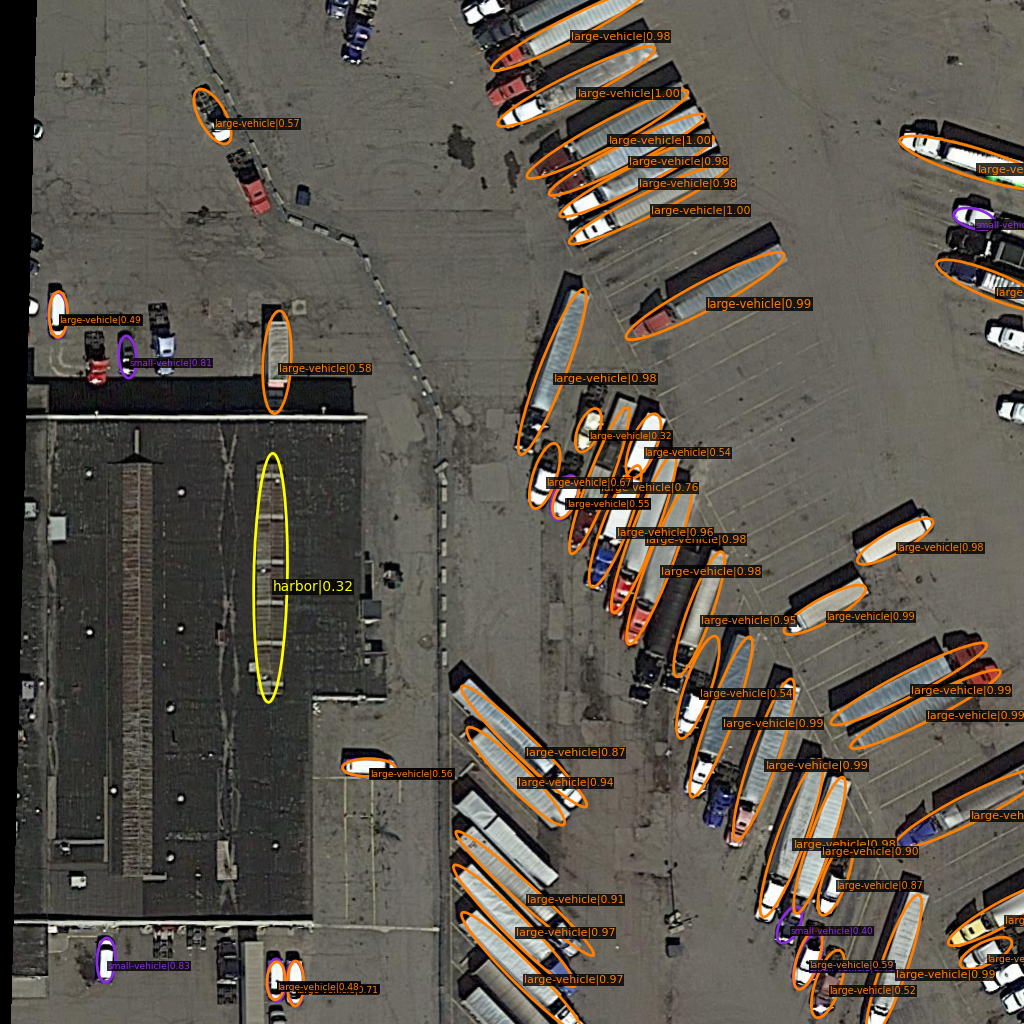}
    \label{fig:dota_gaucho_ellipse1}}
    \subfloat[]{
    \includegraphics[height = 4.1cm, width = 4.1cm]{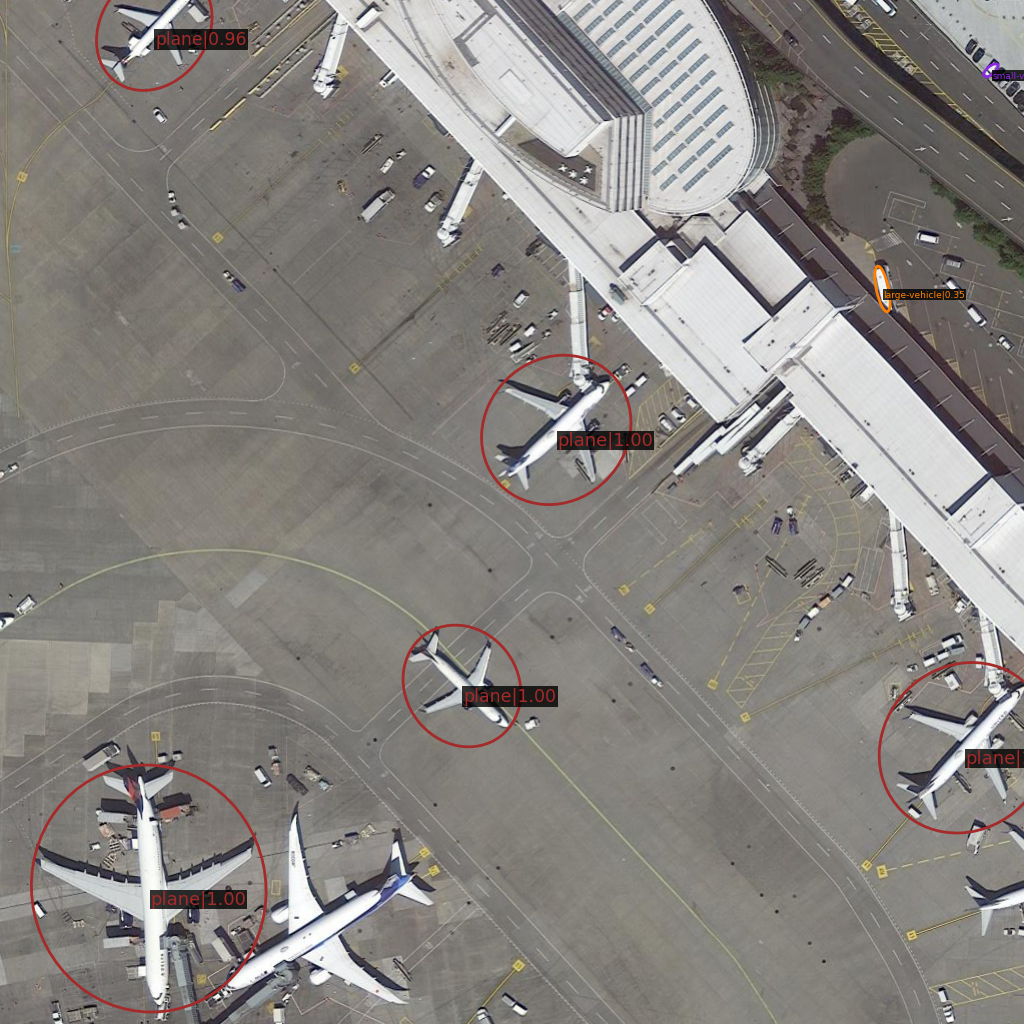}
    \label{fig:dota_gaucho_ellipse2}}
    \subfloat[]{
    \includegraphics[height = 4.1cm, width = 4.1cm]{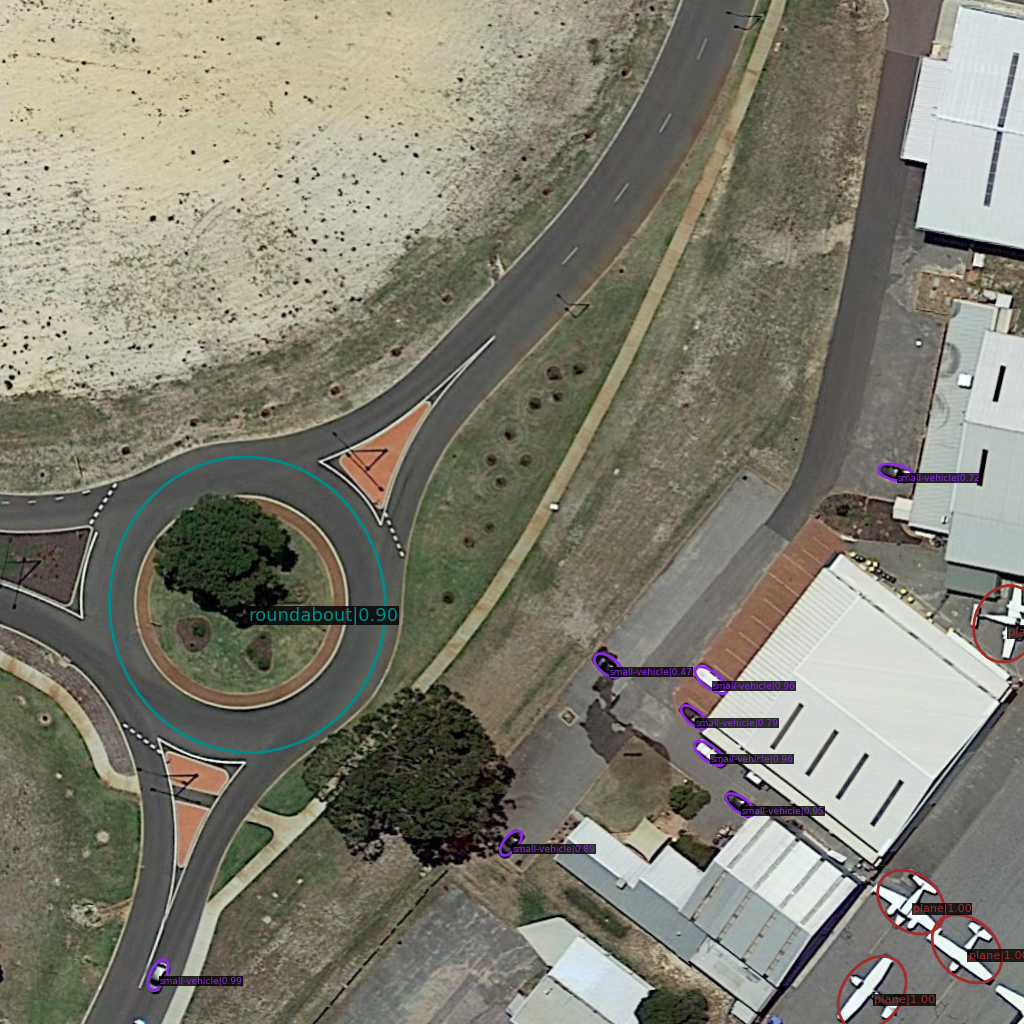}
    \label{fig:dota_gaucho_ellipse3}}
    \subfloat[]{
    \includegraphics[height = 4.1cm, width = 4.1cm]{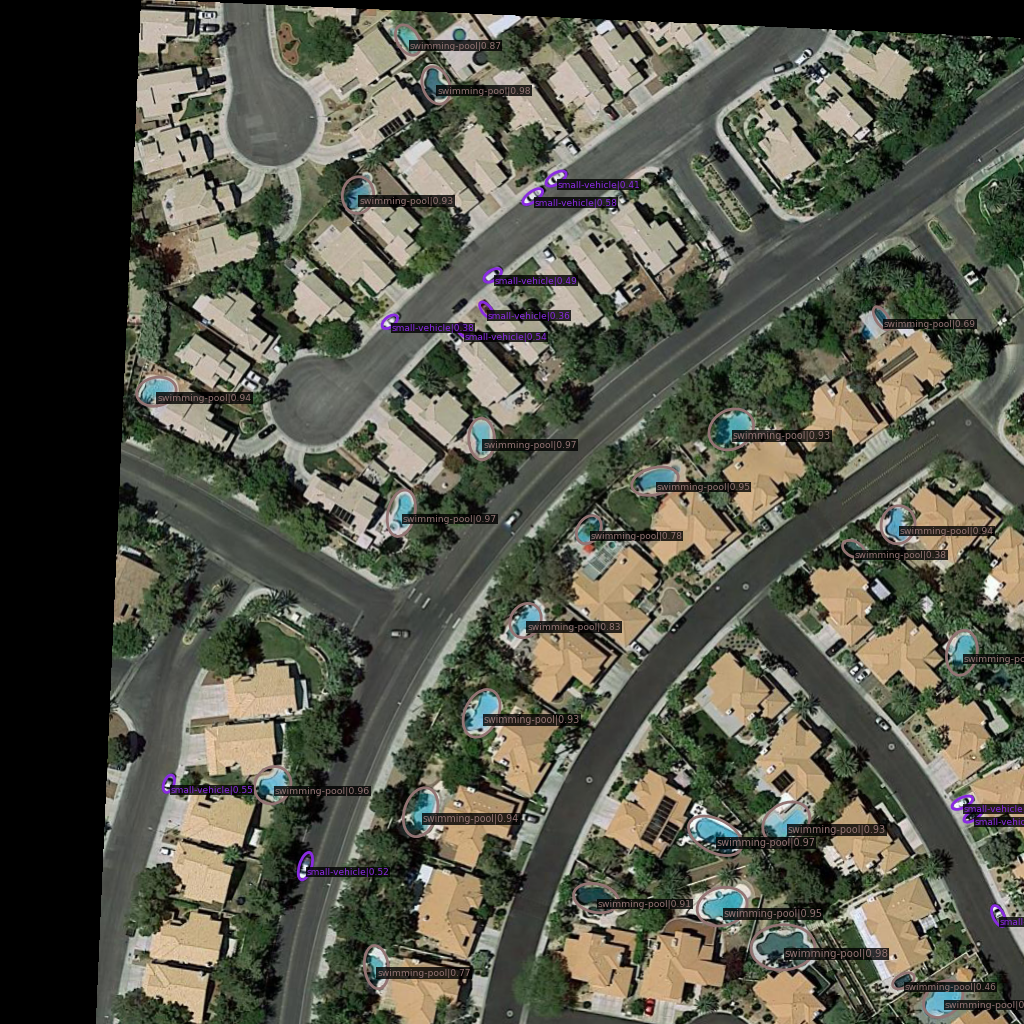}
    \label{fig:dota_gaucho_ellipses4}}
    \caption{Qualitative comparison of FCOS-Gaucho using OBBs (top) and OEs (bottom)  from some images from the DOTA dataset. Better seen zoomed. }
    \label{fig:visual:results_dota}
\end{figure*}

{
    \small
    \bibliographystyle{ieeenat_fullname}
    \bibliography{main}
}


\end{document}